\documentclass[twoside]{article}
\usepackage[accepted]{aistats2026}
\usepackage[round]{natbib}

\usepackage{amsmath,amssymb,amsthm,mathtools}
\usepackage{xurl}      
\usepackage[hidelinks]{hyperref}
\usepackage{enumitem}
\usepackage{subcaption}

\newtheorem{theorem}{Theorem}
\newtheorem{lemma}{Lemma}
\newtheorem{fact}{Fact}
\newtheorem{proposition}{Proposition}

\newtheorem{informaltheorem}{(Informal) Theorem}

\theoremstyle{definition}
\newtheorem{definition}{Definition}

\theoremstyle{definition}
\newtheorem{remark}{Remark}
\newtheorem*{remark*}{Remark}

\usepackage{xcolor}
\usepackage{ifthen}
\newboolean{showtodos}
\setboolean{showtodos}{false} 

\ifthenelse{\boolean{showtodos}}{%
  \usepackage[colorinlistoftodos,textsize=footnotesize]{todonotes}
}{%
  \usepackage[disable]{todonotes}
}

\begin{document}

%

%

\twocolumn[
\aistatstitle{Finite-Time Analysis of Gradient Descent for Shallow Transformers}

\aistatsauthor{ Enes Arda \And Semih Cayci \And  Atilla Eryilmaz }

\aistatsaddress{The Ohio State University \\ \texttt{arda.2@osu.edu} \And  RWTH Aachen University \\ \texttt{cayci@mathc.rwth-aachen.de} \And The Ohio State University \\ \texttt{eryilmaz.2@osu.edu}}]

\begin{abstract}
Understanding why Transformers perform so well remains challenging due to their non-convex optimization landscape. In this work, we analyze a shallow Transformer with $m$ independent heads trained by projected gradient descent in the kernel regime. Our analysis reveals two main findings: (i) the width required for nonasymptotic guarantees scales only \emph{logarithmically} with the sample size $n$, and (ii) the optimization error is \emph{independent} of the sequence length $T$. This contrasts sharply with recurrent architectures, where the optimization error can grow exponentially with $T$. The trade-off is memory: to keep the full context, the Transformer's memory requirement grows with the sequence length. We validate our theoretical results numerically in a teacher--student setting and compare Transformers with recurrent architectures on an autoregressive task.
\end{abstract}

\section{INTRODUCTION} \label{sec:introduction}
Transformers have reshaped modern machine learning, achieving state-of-the-art results in language, vision, and multimodal tasks~%
\citep{vaswani2017attention,devlin2019bert,brown2020language,dosovitskiy2021vit}. Despite this empirical success, a sharp nonasymptotic understanding of their training dynamics remains elusive. Two obstacles stand out: (i) self-attention is highly nonlinear---each output coordinate depends on all inputs through a softmax reweighting---so standard arguments for fully-connected networks or CNNs do not apply directly; and (ii) the training landscape is nonconvex, complicating guarantees for gradient-based methods beyond infinite-width limits.

A growing theory literature has therefore analyzed simplified settings. To make comparisons precise, let us recall the single-head self-attention layer. For a length-$T$ input $X\in\mathbb{R}^{d\times T}$ and weights $W_Q,W_K,W_V\in\mathbb{R}^{m\times d}$, the attention layer outputs
\begin{equation} \label{eq:attn}
    a(X)\;=\;W_VX\,\sigma_s\!\big(\beta X^\top W_K^\top W_Q X\big),
\end{equation}
where $\sigma_s$ denotes the column-wise softmax, and $\beta$ is the softmax scaling. In the original Transformer, one takes $\beta=m^{-1/2}$ to stabilize the dot-product variance~\citep{vaswani2017attention}. However, to make the analysis of \eqref{eq:attn} tractable, prior work often studies variants that take one of the following directions: (i) avoid $\sigma_s$ entirely to obtain linear attention~\citep{katharopoulos2020transformers,choromanski2021rethinking,lu2021softmaxfree,zheng2022linear, zhang2024linear}; (ii) replace $\sigma_s$ with a pointwise activation such as ReLU~\citep{wortsman2023replacing,hron2020infinite}; or (iii) take infinite-width limits under $\beta = m^{-1}$, in which the softmax weights approach uniform distribution and the attention layer reduces to mean-pooling~\citep{wu2023convergence,noci2023shaped,bordelon2024infinite}. While analytically convenient, these modifications limit the scope of the resulting guarantees. 

More recently, \citet{wu2023convergence} established global convergence for shallow Transformers at finite width under the realistic $\beta=m^{-1/2}$ scaling, showing that \emph{quadratic} overparameterization in the sample size $n$ suffices for gradient descent to converge at a \emph{linear} rate. Their proof follows the standard recipe of enforcing strict positive definiteness of a Gram matrix to drive convergence \citep{du2019a,du2019b,allenzhu2019,zou2019,nguyen2021global,nguyen2021tight}. In particular, \citet{wu2023convergence} control the minimum singular value of a feature matrix built from the last hidden layer, which relies on the regime \(d \ge n\), thereby limiting its practicality.

Parallel to these developments for Transformers, recent work has sharpened finite-time analyses for recurrent architectures~\citep{cayci2024rnngd,cayci2024recnpg}. In particular, \citet{cayci2024rnngd} prove nonasymptotic convergence of gradient descent for independently recurrent neural networks (IndRNNs, \citealp{li_independently_2018})
 with width scaling only \emph{logarithmically} in the sample size $n$, using transportation mappings to represent targets in the associated kernel function space. Their bounds also quantify the cost of \emph{long-term dependencies}: beyond an activation-controlled cutoff, both the required width and iteration complexity can grow exponentially with the sequence length~$T$.

In this paper, building on techniques from \citet{cayci2024rnngd}, we analyze a shallow multi-head Transformer trained by gradient descent in the near-initialization regime while preserving the nonlinearity of the attention layer. The main contributions of this work can be highlighted as follows: 

  \textbullet\hskip 6pt\textbf{Preserving attention nonlinearity with multiple heads.} Unlike the approaches discussed above that remove or linearize softmax, or that enforce degeneracy via the scaling $\beta = m^{-1}$, our analysis retains the \emph{genuine attention nonlinearity} and allows independent heads. 
  
  \textbullet\hskip 6pt \textbf{Bypassing positive definiteness of the NTK/Gram.} Unlike the original works in the NTK regime, our analysis \emph{does not} require strict positive definiteness of the NTK/Gram matrix.
  
  \textbullet\hskip 6pt \textbf{Logarithmic overparameterization.} We show that, in our setting, the width required for nonasymptotic training guarantees scales only \emph{logarithmically} with the sample size $n$. To the best of our knowledge, this has not been shown for Transformers.
  
  \textbullet\hskip 6pt \textbf{Optimization error independent of input sequence length $T$.} Our convergence bound does \emph{not} degrade with the sequence length $T$. This contrasts sharply with RNN analyses, where the dependence on $T$ can be exponential under long-term dependencies. The trade-off is memory: retaining full context implies that the Transformer's memory footprint grows with $T$ due to attention over all tokens.
  
  \textbullet\hskip 6pt \textbf{Experimental validation.} With controlled experiments in teacher--student settings, we validate the predicted scaling laws for our Transformer model and contrast Transformers with IndRNNs in their handling of long-term dependencies and memory complexity.

\paragraph{Technical challenges and ideas.}
(i) We preserve the genuine softmax attention nonlinearity while avoiding any uniformly positive definite Gram/NTK assumption by working in a near-initialization regime, where the linearization error can be controlled under an explicit width requirement (Lemma~\ref{lem:linearization}). (ii) Within this NTK viewpoint, transportation mappings (Section~\ref{sec:transportation-maps}) provide a
convenient representation of targets in the Transformer-NTK RKHS and yield an approximation bound with only logarithmic dependence on the sample size (Lemma~\ref{lem:approximation}). (iii) To obtain bounds independent of the sequence length $T$, we exploit the covariance structure
induced by the softmax Jacobian to bound attention gradient/Hessian norms uniformly in $T$
(Remark~\ref{rem:cov-bound}); this feeds into $T$-independent local Lipschitz/smoothness constants
(Lemma~\ref{lem:h-local-Lipschitz-smooth}) and ultimately yields the $T$-independent optimization error (Theorem~\ref{thm:projgd}). 

All proofs of technical results are deferred to the appendix, which also provides additional details.

\section{RELATED WORK}

\paragraph{Neural Tangent Kernel (NTK) analysis.}
The NTK framework linearizes networks around initialization and connects gradient descent to kernel gradient flow, yielding optimization and generalization guarantees in overparameterized regimes \citep{jacot2018ntk,lee2019wide, du2019b, chizat2019lazy}. For CNNs, NTK limits capture locality and weight sharing \citep{garriga2019deepgp,novak2019bayesiancnn}, and the convolutional NTK admits exact computation \citep{arora2019cntk}. For RNNs, the Recurrent NTK defines sequence kernels that aggregate contributions across time \citep{alemohammad2021rntk}. 

\paragraph{Asymptotic theory for Transformers.}
From a representation-theoretic perspective, Transformers are universal approximators of permutation-equivariant sequence-to-sequence maps, and, when equipped with trainable positional encodings, of general continuous sequence-to-sequence functions on compact domains \citep{yun2020universal}.
At infinite width, attention admits NTK descriptions under suitable scalings: \citet{hron2020infinite} derive attention NTK limits, showing that under the common $\beta=m^{-1/2}$, letting the number of heads grow results in a deterministic NTK. Under $\beta=m^{-1}$, deterministic kernels can exist but may degenerate toward mean-pooling or require $W_Q=W_K$ \citep{yang2020tensor2}. \citet{wu2023convergence} complement their finite‑width convergence proof with an NTK analysis for shallow Transformers with the degenerate scaling $\beta=m^{-1}$. 

\paragraph{Nonasymptotic theory for Transformers.}
\citet{wu2023convergence} prove global convergence of a shallow Transformer while preserving the softmax attention nonlinearity and carefully diagnosing scaling/initialization choices.
\citet{deora2024opt} derive finite-time training and generalization bounds for gradient descent on a single-layer multi-head self-attention model in binary classification, where they obtain polylogarithmic head requirements under realizability and initialization conditions.
Complementary results capture in‑context learning dynamics with true softmax: \citet{huang2024incontext} give finite‑time guarantees and a phase‑wise view of attention; \citet{song2024dynamics} show convergence under suitable initialization and reveal that softmax can introduce spurious local minima whereas a Gaussian attention smooths the landscape. Simplified settings using linear attention establish convergence for linear predictors learned in‑context \citep{zhang2024linear}. Beyond linear tasks, \citet{shen2025convergence} obtain linear‑rate convergence to the Bayes classifier for in‑context classification with a shallow multi‑head Transformer. Further optimization views sharpen what attention learns: \citet{tarzanagh2023maxmargin} show that gradient descent on a single attention layer aligns with a max‑margin (SVM‑like) token‑selection rule. \citet{vasudeva2025imp} study implicit bias and convergence rates for single-layer self-attention in binary classification. \citet{nichani2024causal} analyze a two‑layer attention‑only Transformer and prove that gradient descent recovers the latent causal structure of the in-context learning problem.

\paragraph{Comparison to \citet{wu2023convergence}.}
Our work is closest in spirit to \citet{wu2023convergence}, which also studies shallow Transformers with
genuine softmax attention under a regression setting. The main differences are:
\textbf{(i) Regime:} \citet{wu2023convergence} prove global convergence via a well-conditioned last-layer
feature/Gram matrix, whereas we work in a projected near-initialization NTK regime and do not assume
a positive definite Gram matrix.
\textbf{(ii) Rates/width:} their result achieves a linear rate with polynomial overparameterization in $n$ \citep[Proposition~1]{wu2023convergence}, while our bounds have the canonical $1/\sqrt{\tau}$ optimization rate and require only logarithmic dependence on $n$ in the width (Theorem~\ref{thm:projgd}).
\textbf{(iii) Sequence length:} our optimization constants are explicitly independent of $T$, while their
conditions involve $T$-dependent quantities.
\textbf{(iv) NTK features:} they complement their analysis with an NTK study under $1/m$-type scaling, which yields a vanishing attention component \citep[Lemma~1]{wu2023convergence}; whereas our NTK decomposition retains a nontrivial attention block (Lemma~\ref{lem:ntk}).

\paragraph{Notation.} For $k\in\mathbb{N}$ we write $[k]=\{1,2,\dots,k\}$. For matrices $A,B$ we write $\langle A,B\rangle_F=\mathrm{tr}(A^\top B)$ and $\|A\|_F$ for the Frobenius norm. The probability simplex in $\mathbb{R}^T$ is $\Delta_T=\{\,\alpha\in\mathbb{R}^T:\alpha_t\ge0,\ \sum_{t=1}^T\alpha_t=1\,\}$ and $\mathbb{R}_+=(0,\infty)$. For a vector $x$, $\operatorname{diag}(x)$ denotes the diagonal matrix with diagonal $x$. For matrices $A_1,\ldots,A_r$, $\operatorname{blkdiag}(A_1,\ldots,A_r)$ denotes the block-diagonal matrix with diagonal blocks $A_1,\ldots,A_r$ and zero elsewhere. For $A\in\mathbb{R}^{m\times n}$, $\operatorname{vec}(A)\in\mathbb{R}^{mn}$ denotes the column-wise vectorization of $A$. The notation $\lesssim_{a,b}$ omits the constants depending on the parameters $a,b$. For vectors $x,y$, $x\le y$ denotes element-wise comparison. $\mathrm{Unif}(V)$ denotes the uniform distribution over the set $V$, and $\mathrm{Rad}(\alpha)$ denotes $\mathrm{Unif}(\{-\alpha, \alpha\})$.

\section{PROBLEM SETUP} \label{sec:setup}

We start by specifying the objective, introduce the shallow multi-head Transformer studied in this work, and formalize the optimization algorithms analyzed throughout the paper.

\subsection{Data}
Each sample $X=(X_1,\dots,X_T)\in\mathbb{R}^{d\times T}$ is a sequence of $d$-dimensional tokens coming from a compact set 
\[\mathcal{X} \;\coloneqq\; \bigl\{\,X \in \mathbb{R}^{d \times T} \;:\; \max_{t \in [T]} \|X_t\|_2 \le 1 \,\bigr\}.\]
The label is a scalar $y\in\mathbb{R}$. We observe samples from a dataset
\[
\mathcal{D}=\bigl\{(X^{(j)},y^{(j)})\bigr\}_{j=1}^n\ \subset\ \mathcal{X}\times\mathbb{R}
\]
that were generated by a function $f^\star$. In other words, $f^\star(X^{(j)})=y^{(j)}$ for every $j \in [n]$. Our goal will be to approximate the function $f^\star$ with the Transformer model defined in the next section.

\subsection{Transformer Architecture}
\label{sec:arch}

Fix $X \in \mathcal{X}$. Since each training sample has a scalar target, we employ pooled self-attention driven by a single query $q_X\in\mathbb{R}^d$ with $\|q_X\|_2 \leq 1$. The query $q_X$ is not trainable: it may be a fixed vector (\texttt{[CLS]} as in BERT \citep{devlin2019bert}, or ViT’s class token \citep{dosovitskiy2021vit}), a deterministic function of the input (e.g., last-token pooling $q_X=X_T$ for next-token prediction \citep{radford2019gpt2}), or a pre-specified seed query used to summarize high-dimensional inputs \citep{jaegle2021perceiverio}.

Let $m \in \mathbb{N}$ be the width of our Transformer. For \(i \in [m]\) and trainable \(W_i\!\in\!\mathbb{R}^{d\times d}\), \(U_i\!\in\!\mathbb{R}^{d}\), \(c_i\!\in\!\mathbb{R}\), we consider the following model.


{\textbullet\hskip 6pt \textbf{Attention layer.}}
\begin{equation} \label{eq:our-attn}
    a(X; W_i) \coloneqq X\,\sigma_s\!\big(X^\top W_i q_X\big) \in \mathbb{R}^d,
\end{equation}
where \(\sigma_s:\mathbb{R}^T\!\to\!\Delta_T\) is the softmax function: $(\sigma_s(z))_t = \tfrac{\exp(z_t)}{\sum_{s=1}^T \exp(z_s)}.$

{\textbullet\hskip 6pt \textbf{Feed-forward layer.}}
\begin{equation}
    h(X;\theta_i) \coloneqq \sigma\!\big(U_i^\top a(X; W_i)\big) \in \mathbb{R},
\end{equation}
where \(\theta_i \coloneqq (U_i,\operatorname{vec}(W_i))\) and the activation \(\sigma:\mathbb{R}\!\to\!\mathbb{R}\) satisfies uniform bounds
\[
|\sigma(u)|\le \sigma_0,\; |\sigma'(u)|\le \sigma_1,\; |\sigma''(u)|\le \sigma_2\quad\forall u\in\mathbb{R}.
\]
We collect these bounds as $\bar{\sigma}=(\sigma_0, \sigma_1, \sigma_2) \in \mathbb{R}^3_+$.

{\textbullet\hskip 6pt \textbf{Linear layer.}}
\begin{equation}
    f(X;\varphi) \coloneqq \frac{1}{\sqrt{m}}\sum_{i=1}^m c_i\,h(X;\theta_i) \in \mathbb{R},
\end{equation}
where \(\varphi_i \coloneqq (c_i,U_i,\operatorname{vec}(W_i))\) and \(\varphi := (\varphi_1, \dotsc, \varphi_m)\).

\begin{remark}\label{rem:qk-scaling}
Following prior analyses \citep{wu2023convergence, deora2024opt, nichani2024causal, Oymak2023attentionPromptTuning},
we simplify the standard $(W_Q,W_K,W_V)$ parameterization in \eqref{eq:attn}.
For pooled attention with a single query, the usual logits can be written as
$\beta\,(W_K X)^\top (W_Q q_X)\in\mathbb{R}^T$ and hence
$\beta\,(W_K X)^\top (W_Q q_X)=X^\top(\beta W_K^\top W_Q)q_X$.
Thus, collapsing $W_Q$ and $W_K$ into a single matrix in \eqref{eq:our-attn} amounts to setting
$W_i:=\beta W_{K,i}^\top W_{Q,i}$.
By initializing $W_i$ with bounded-variance entries, the scaling $\beta = 1/\sqrt{m}$ is absorbed into the
initialization and does not appear explicitly in \eqref{eq:our-attn}. In contrast, $1/m$ scaling would drive logits toward zero and make the softmax weights nearly
uniform, degenerating attention toward mean pooling (cf.\ Section~\ref{sec:introduction}).
Finally, in this one-hidden-layer setting we absorb the value/output matrices into the subsequent linear
map $U_i$.
\end{remark}

\begin{remark}
In standard multi-head attention one concatenates $h$ attention heads into $A(X)=(a(X;W_1),\ldots,a(X;W_h))\in\mathbb{R}^{hd}$ and applies a feed-forward layer with a weight matrix $V\in\mathbb{R}^{m\times hd}$, resulting in $\sigma(VA(X))$.
Our architecture is equivalent to this with a block-diagonal weight matrix with one neuron per head, i.e., $h=m$ and $V=\operatorname{blkdiag}(U_1^\top,\ldots,U_m^\top)$. This independent head constraint acts as a structural regularizer: it preserves per-head interpretability, reduces cross-head interference, and restores a deterministic NTK, thereby simplifying the analysis without requiring any of the architectural changes mentioned in Section~\ref{sec:introduction}.
\end{remark}

\subsection{Algorithms for Empirical Risk Minimization}\label{sec:algorithms}
For a given dataset $\mathcal{D}$, parameters $\varphi$ and the predictor $f(\cdot;\varphi)$, we analyze the empirical mean-squared error (MSE) loss, which is expressed as
\[
\widehat{\mathcal{L}}_n(\varphi):=\frac{1}{n}\sum_{j=1}^n\bigl(f(X^{(j)};\varphi)-y^{(j)}\bigr)^2.
\]
We optimize $\widehat{\mathcal{L}}_n$ by projected methods onto the feasible set $\Omega$. Let $\ell_j(\varphi):=\big(f(X^{(j)};\varphi)-y^{(j)}\big)^2$ and, for any index set $B\subseteq [n]$,
$\widehat{\mathcal L}_{B}(\varphi)\;:=\;\frac{1}{|B|}\sum_{j\in B}\ell_j(\varphi).$
Given a step size $\eta$ and mini‑batches $\{B_s\}_{s\ge 0}$, the iterate updates are, for $s=0,1,2,\dots$:
\begin{align*}
\varphi^{(s+1)} &:= \Pi_{\Omega}\left(\varphi^{(s)} - \eta\,\nabla_\varphi \widehat{\mathcal L}_{B_s}\!\big(\varphi^{(s)}\big)\right), 
\end{align*}
where $\Pi_{\Omega}$ denotes projection onto the set $\Omega$. We analyze (i) \textsc{ProjGD} with $B_s=[n]$ for all $s$, and (ii) \textsc{ProjSGD} with $B_s=\{J_s\}$ with $J_s\stackrel{\mathrm{i.i.d.}}{\sim}\mathrm{Unif}([n])$.

\paragraph{Why projection?}
Our analysis is local around the random initialization: all bounds are proved on the neighborhood
$\Omega_\rho$ (Definition~\ref{def:parameter-set}), where (i) the NTK linearization error is uniformly
controlled (Lemma~\ref{lem:linearization}) and (ii) the change-of-feature terms in the optimization
argument can be bounded using local Lipschitz/smoothness constants (Lemma~\ref{lem:h-local-Lipschitz-smooth}).
Projecting the iterates onto $\Omega_\rho$ is therefore a simple analytical device that enforces this
near-initialization regime. In practice, other mechanisms that keep parameters close to initialization
(e.g., weight decay, explicit norm regularization, or early stopping) can play a similar role, and in
Section~\ref{sec:exp} we also report experiments trained without projection.

\section{OVERVIEW OF MAIN RESULTS}\label{sec:main-results-informal}

Below we present our main results informally. Formal statements are given in Section~\ref{sec:main-results}, with detailed proofs deferred to Appendix~\ref{sec:app-missing-proofs}.
\begin{informaltheorem}
Let $f^\star \in \mathcal{F}_{\bar\nu}$ be the target function (formalized in Section~\ref{sec:transportation-maps}).
After $\tau$ steps of \textsc{ProjGD} with an appropriate step size and projection radius,
the average iterate $\bar\varphi^{(\tau)} := \frac{1}{\tau}\sum_{s=0}^{\tau-1}\varphi^{(s)}$ satisfies,
with probability at least $1-\delta-\delta'$ over the random initialization
(introduced in Section~\ref{sec:init}),
\begin{align*}
\widehat{\mathcal{L}}_n\!\left(\bar\varphi^{(\tau)}\right)
\lesssim
  \underbrace{\frac{D^2}{\sqrt{\tau}}}_{\substack{\text{optimization}\\\text{error}}}
  +
  \underbrace{\sqrt{\frac{D\log(n/\delta)}{m}}}_{\substack{\text{approximation}\\\text{error}}}
  +
  \underbrace{\sqrt{\frac{D^3}{m}}}_{\substack{\text{linearization}\\\text{error}}}
\end{align*}
where $D := d + \log(m/\delta')$.
\end{informaltheorem}

For \textsc{ProjSGD}, the same bound holds for $\mathbb{E}[\widehat{\mathcal L}_n(\bar\varphi^{(\tau)})\,|\,\varphi^{(0)}]$ with the same rates (cf. Proposition~\ref{thm:projsgd}). 

We highlight the following features:

\textbf{\textbullet \hskip 6pt Non-asymptotic in time and width.}
The training loss admits a finite-time optimization--approximation--linearization decomposition, with
the canonical $1/\sqrt{\tau}$ optimization rate and $1/\sqrt{m}$ width dependence.
This decomposition holds for any $\tau, m \geq 1$, with no hidden asymptotic regime.

\textbf{\textbullet \hskip 6pt Logarithmic in sample size.}
The sample size enters only through the approximation error (Lemma~\ref{lem:approximation}). To make this term at most $\varepsilon$ it suffices to take $m \gtrsim \varepsilon^{-2}\log(n/\delta)$; in particular, $m$ needs to grow only \emph{logarithmically} with $n$.

\textbf{\textbullet \hskip 6pt No explicit dependence on sequence length.}
Thanks to the attention gradients being uniformly bounded in $T$ via a softmax-covariance structure (Remark~\ref{rem:cov-bound}), we obtain an optimization error independent of the sequence length $T$.

\section{TRANSFORMER NTK ANALYSIS} \label{sec:ntk}

In this section, we analyze the gradients of the Transformer architecture from Section~\ref{sec:arch}, specify the initialization, derive the associated NTK, and introduce transportation mappings. These components define the function class used to establish our main results in Section~\ref{sec:main-results}.

\subsection{Transformer Gradient Analysis} \label{sec:gradient-analysis}
\begin{lemma}[Gradients] \label{lem:grads}
    Fix \(i\in[m]\) and write \(a_i:=a(X;W_i)\).
    Then
    \begin{align}
    \frac{\partial f(X;\varphi)}{\partial c_i}
    &= m^{-1/2}\, h(X;\theta_i),\\
    \nabla_{U_i} f(X;\varphi)
    &= m^{-1/2}\, c_i\, \sigma'\!\big(U_i^\top a_i\big)\, a_i,\\
    \nabla_{W_i} f(X;\varphi)
    &= m^{-1/2}\, c_i\, \sigma'\!\big(U_i^\top a_i\big)\, \big(M_i U_i\big)\, q_X^\top, \label{eq:grad-w}
    \end{align}
    where
    \begin{equation}
    M_i \;:=\; M(X;W_i) \;=\; X\, J_s\!\big(X^\top W_i q_X\big)\, X^\top,
    \end{equation}
    and
       $ J_s(z) \;=\; \operatorname{diag}\!\big(\sigma_s(z)\big) - \sigma_s(z)\sigma_s(z)^\top$
    is the Jacobian matrix of softmax.
\end{lemma}

\begin{remark} \label{rem:cov-bound}
With $\alpha=\sigma_s(X^\top W_i q_X)\in\Delta_T$ and $\mu=\sum_t\alpha_t X_t$, notice that $M(X;W_i)=\sum_{t=1}^T \alpha_t X_tX_t^\top-\mu\mu^\top$ is a softmax‑weighted covariance matrix. This is the only term in Lemma~\ref{lem:grads} that involves a $T$-dependence. Since $\mathrm{tr}\,M(X;W_i)\le 1$ for all $X\in\mathcal X$, we obtain $\|M(X;W_i)\|_F\le 1$ \emph{uniformly in $T$}. As a consequence of this key property, we observe that $\|\nabla_{W_i} f(X;\varphi)\|_F$ is independent of $T$. In contrast, for RNNs the analogous gradient contains an explicit sum over $T$ recurrent steps, leading to constants that can grow with $T$ \citep[Prop.~3.1]{cayci2024rnngd}. This structural difference underlies our $T$‑independent optimization bounds in Section~\ref{sec:main-results}. See Appendix~\ref{sec:useful-facts} for the detailed derivation.
\end{remark}

\begin{remark}
Similarly, we have $a(X;W_i)=\sum_{t=1}^T \alpha_t X_t$, which is a convex combination of the tokens $\{X_t\}_{t=1}^T$. Hence, by the triangle inequality, $\|a(X;W_i)\|_2 \;\le\; \sum_{t=1}^T \alpha_t \|X_t\|_2 \;\le\; \sum_{t=1}^T \alpha_t \;=\; 1$. This allows us to bound $\|\nabla_{U_i} f(X;\varphi)\|_2$, and together with the previous remark, these bounds yield the continuity and smoothness results stated in Lemma~\ref{lem:h-local-Lipschitz-smooth}.
\end{remark}

\subsection{Initialization} \label{sec:init}
As noted in previous work \citep{wu2023convergence, cayci2024rnngd, lee2019wide, jacot2018ntk, chizat2019lazy}, initialization plays an important role in the performance of gradient descent in the kernel regime. Following \citep{Bai2020Beyond, cayci2024rnngd, chizat2019lazy}, we adopt a symmetric random initialization that centers the model at zero output, which will be important in the convergence analysis. Assume $m$ is even. For $i\in[m/2]$ initialize
\[
\begin{aligned}
W_i^{(0)} &\in \mathbb{R}^{d\times d}\ \text{with i.i.d.\ entries } \mathcal N(0,1),\\
U_i^{(0)} &\sim \mathcal N(0,I_d), \qquad c_i^{(0)} \sim \mathrm{Rad}(1).
\end{aligned}
\]
and set
\[
W^{(0)}_{i+m/2}=W^{(0)}_i,\quad U^{(0)}_{i+m/2}=U^{(0)}_i,\quad c^{(0)}_{i+m/2}=-\,c_i^{(0)},
\]
with all draws mutually independent. We will denote this random initialization by $\varphi^{(0)}$. 

\begin{remark}
Writing $h_i(X)=h(X;\theta_i^{(0)})$, note that the pairing above gives $h_{i+m/2}(X)=h_i(X)$, and therefore $f\bigl(X;\varphi^{(0)}\bigr) = 0$ for all $X \in \mathcal{X}$.
This fact will be critical to achieve the error bounds in Lemmas~\ref{lem:linearization} and \ref{lem:approximation}.
\end{remark}

Because $U_i^{(0)}$ is a sequence of sub-Gaussian vectors, the running maximum of its norm stays bounded with high probability, which allows us to bound the norm of the attention weight gradient in \eqref{eq:grad-w} with high probability. For this, we define the following event.

\begin{definition}\label{def:good-U-init}
For $\delta'\in(0,1)$, define
\[
\mathcal E_{U,m}(\delta'):=\Big\{\,\max_{i\in[m]}\|U_i^{(0)}\|_2 \;\le\; \sqrt d + \sqrt{2\log(m/2\delta')}\,\Big\}.
\]
By Gaussian concentration for the $1$-Lipschitz map $u\mapsto\|u\|_2$ \citep[Thm.~5.6]{boucheron2013concentration} and a union bound over $i\in[m/2]$, we obtain 
$\mathbb P\big(\mathcal E_{U,m}(\delta')\big)\;\ge\;1-\delta'.$
\end{definition}

Throughout the paper, we will work in the following parameter set centered around the initialization $\varphi^{(0)}$. 
\begin{definition} \label{def:parameter-set}
    Let $\rho=(\rho_c,\rho_u,\rho_w)\in\mathbb{R}_+^3$. For $i\in[m]$ define
    \begin{align*}
    \mathcal{C}_{\rho_c,i}
    &:= \{\, c\in\mathbb{R} : |c-c_i^{(0)}| \le \rho_c/\sqrt{m} \,\},\\
    \mathcal{U}_{\rho_u,i}
    &:= \{\, U\in\mathbb{R}^d : \|U-U_i^{(0)}\|_2 \le \rho_u/\sqrt{m} \,\},\\
    \mathcal{W}_{\rho_w,i}
    &:= \{\, W\in\mathbb{R}^{d\times d} : \|W-W_i^{(0)}\|_F \le \rho_w/\sqrt{m} \,\},
    \end{align*}
    and the product set
    \[
    \Omega_{\rho,i} \,:=\, \mathcal{C}_{\rho_c,i}\times \mathcal{U}_{\rho_u,i}\times \mathcal{W}_{\rho_w,i},
    \quad\,
    \Omega_\rho \,:=\, \prod_{i=1}^m \Omega_{\rho,i}.
    \]
\end{definition}

Notice that on $\mathcal E_{U,m}(\delta')$ we also have for all $i\in[m]$ and all $U_i\in\mathcal U_{\rho_u,i}$
\begin{equation} \label{eq:U-norm-bound}
\|U_i\|_2 \le \sqrt d + \sqrt{2\log(m/2\delta')}+ \frac{\rho_u}{\sqrt m}=: B_{U,m}(\delta').
\end{equation}

This allows us to establish local Lipschitz continuity and smoothness, as formalized in the following lemma, which in turn enables control of the optimization error in Theorem~\ref{thm:projgd}.
\begin{lemma} \label{lem:h-local-Lipschitz-smooth}
    On $\mathcal E_{U,m}(\delta')$, the mapping $\theta_i \mapsto h(X;\theta_i)$ is $L_{1,m}(\delta')$-Lipschitz and $L_{2,m}(\delta')$-smooth in $\Omega_\rho$, where
        $L_{1,m}(\delta') := \sigma_1\sqrt{1 + B_{U,m}(\delta')^2},$
    and
    \[
        L_{2,m}(\delta') := \sigma_2(1 + B_{U,m}(\delta')^2) + 8\sigma_1\sqrt{1 + B_{U,m}(\delta')^2}.
    \]
\end{lemma}

Next, we define the linear model, which will be useful in the analysis to characterize the linearization error.
\begin{definition}[Linear model]
    Define the function
    \[
    f_{\mathrm{lin}}(X;\varphi)
    := f(X;\varphi^{(0)}) \;+\;
    \big\langle \nabla_{\varphi} f(X;\varphi^{(0)}),\, \varphi-\varphi^{(0)} \big\rangle,
    \]
    which is the first-order approximation of $f$ at initialization. Thanks to symmetric initialization, we get
    \[
    f_{\mathrm{lin}}(X;\varphi)= \big\langle \nabla_{\varphi} f(X;\varphi^{(0)}),\, \varphi-\varphi^{(0)} \big\rangle.
    \]
\end{definition}

Near initialization, higher-order errors vanish, so the linear model closely approximates the original model. Indeed, inside the neighborhood $\Omega_\rho$, we get the following uniform linearization error bound.

\begin{lemma}[Linearization error] \label{lem:linearization}
Assume \(\varphi \in \Omega_\rho\) and define $\varepsilon_{\mathrm{lin}}:=\sup_{j\in[n]} \big| f(X^{(j)};\varphi) - f_{\mathrm{lin}}(X^{(j)};\varphi) \big|$. Conditioned on the event $\mathcal{E}_{U,m}(\delta')$, we have
\[
\varepsilon_{\mathrm{lin}}
\leq \tfrac{1}{\sqrt{m}}\Bigl(L_{1,m}(\delta')\rho_c\sqrt{\rho_w^2+\rho_u^2}
+ L_{2,m}(\delta')(\rho_w^2 + \rho_u^2)\Bigr),
\]
where the upper bound is denoted as $B_{\mathrm{lin},m}(\delta').$
\end{lemma}

\subsection{Transformer NTK} \label{sec:ntk-subsec}
For our convergence analysis, we work in the NTK regime, where training is well-approximated by linearization around initialization. Concretely, for a given random initialization \(\varphi^{(0)}\), the Neural Tangent Kernel (NTK) is defined as
\[
K(X,X')\coloneqq \lim_{m\to\infty}
\Big\langle \nabla_{\varphi} f\!\big(X;\varphi^{(0)}\big),\,
           \nabla_{\varphi} f\!\big(X';\varphi^{(0)}\big)\Big\rangle.
\]
Viewing $\nabla_{\varphi} f(X;\varphi^{(0)})$ as a feature map for the input $X$, $K$ is the kernel induced by these features in the infinite-width limit. In this regime, gradient descent on $f$ linearizes around $\varphi^{(0)}$ and corresponds to kernel regression with kernel $K$.

Since $\nabla_{\varphi} f(X;\varphi^{(0)})$ is block-separable across $(c,U,W)$, we investigate the features coming from these parameters separately. Using the gradient identities in Lemma~\ref{lem:grads}, the Transformer NTK decomposes into the following blocks.

\begin{lemma}[NTK decomposition]\label{lem:ntk}
Let $a=a(X;W)$, $a'=a(X';W)$, $M=X\,J_s(X^\top W q_X)\,X^\top$, and
$M'=X'\,J_s(X'^\top W q_{X'})\,X'^\top$. Then, with expectation taken over the random
initialization $(c,U,\operatorname{vec}(W)) \sim \varphi^{(0)}_1$,
\begin{align*}
K_c(X,X') &= \mathbb{E}\!\left[\,\sigma(U^\top a)\,\sigma(U^\top a')\,\right],\\
K_u(X,X') &= \mathbb{E}\!\left[\,\sigma'(U^\top a)\,\sigma'(U^\top a')\,\langle a,a'\rangle\,\right],\\
K_w(X,X') &= \langle q_X,q_{X'}\rangle\mathbb{E}\!\left[\sigma'(U^\top\!a)\sigma'(U^\top\!a')\,U^\top\!MM'U\right]
\end{align*}
and the NTK decomposes as $K=K_c+K_u+K_w$.
\end{lemma}

\begin{remark}
    Unlike the NTK analysis in \citep[Lemma~1]{wu2023convergence}, here the attention‑weight component $K_w$ is \emph{nonzero}. This is a consequence of not using the degenerate $m^{-1}$ scaling and keeping a random attention weight in \eqref{eq:our-attn}, which preserves attention features in the limit. Moreover, with multiple independent heads, the empirical NTK converges to the deterministic kernel $K$ above, solving the limit kernel problem of a single-head attention mentioned in \citep{hron2020infinite}.
\end{remark}

\subsection{Transportation Mappings} \label{sec:transportation-maps}
Let $\varphi_0 := (c,U,\operatorname{vec}(W)) \sim \varphi^{(0)}_1$. It can be seen that the NTK in Lemma~\ref{lem:ntk} is a Mercer kernel constructed with the following random feature maps:
\begin{align*}
\phi_c(X) &:= \sigma\!\big(U^\top a(X;W)\big)\in\mathbb{R},\\
\phi_u(X) &:= \sigma'\!\big(U^\top a(X;W)\big)\,a(X;W)\in\mathbb{R}^d,\\
\phi_w(X) &:= \sigma'\!\big(U^\top a(X;W)\big)\,(M(X;W)U)\,q_X^\top\in\mathbb{R}^{d\times d}.
\end{align*}
In order to characterize the Reproducing Kernel Hilbert Space (RKHS) associated with the NTK, we will use the concept of transportation mappings \citep{ji2019ntktransport}.

Let $v:\mathbb{R}^{d^2+d+1}\to \mathbb{R}\times\mathbb{R}^d\times\mathbb{R}^{d\times d}$ be a measurable function with $v(\varphi_0) = (v_c(\varphi_0),v_u(\varphi_0),v_w(\varphi_0))$ such that $\mathbb{E}\!\Big[\,|v_c(\varphi_0)|^2+\|v_u(\varphi_0)\|_2^2+\|v_w(\varphi_0)\|_F^2\,\Big]<\infty$. Define the function class
\begin{multline*}
\mathcal{F}
:=\Big\{\tilde f(\cdot;v):
\mathbb{E}|v_c(\varphi_0)|^2<\infty, \;
\mathbb{E}\|v_u(\varphi_0)\|_2^2<\infty,\\
\mathbb{E}\|v_w(\varphi_0)\|_F^2<\infty\Big\},
\end{multline*}
where 
\begin{multline} \label{eq:f_tilde_def}
\tilde f(X;v) := \mathbb{E}\!\big[\phi_c(X)\,v_c(\varphi_0)\big] + \mathbb{E}\!\big[\langle \phi_u(X), v_u(\varphi_0)\rangle\big] \\
 + \mathbb{E}\!\big[\langle \phi_w(X), v_w(\varphi_0)\rangle_F\big]
\end{multline}
and equip \(\mathcal{F}\) with the inner product
\[
\langle \tilde f(\cdot;v),\tilde f(\cdot;v')\rangle_{\mathcal{F}}
:=\mathbb{E}\!\Big[
v_c\,v_c' + \langle v_u,v_u'\rangle + \langle v_w,v_w'\rangle_F\Big].
\]

Notice that for a fixed \(X' \in \mathcal{X}\), one can set
\(v'(\varphi_0):=\left(\phi_c(X'), \phi_u(X'), \phi_w(X')\right)\) to get
\begin{multline*}
\tilde f(X; v')
= \mathbb{E}\!\big[\phi_c(X)\phi_c(X')\big] + \mathbb{E}\!\big[\langle \phi_u(X), \phi_u(X')\rangle\big] \\
+ \mathbb{E}\!\big[\langle \phi_w(X), \phi_w(X')\rangle_F\big]
= K(X,X').
\end{multline*}
In other words, we have $K(\cdot,X')\in\mathcal{F}$ for every $X' \in \mathcal{X}$, and
\(\langle \tilde f(\cdot;v),K(\cdot,X')\rangle_{\mathcal{F}}=\tilde f(X';v)\).
Hence, $K$ has the \emph{reproducing property} on $\mathcal{F}$, which implies that the completion
\(\overline{\mathcal{F}}\) is the RKHS associated with the Transformer NTK \(K\).

We will consider the norm-constrained functions from this class. In particular, for a given \(\bar\nu=(\bar\nu_c,\bar\nu_u,\bar\nu_w) \in \mathbb{R}_+^3\), the function class of interest is defined as
\[
\mathcal{F}_{\bar\nu} := \big\{ \tilde f(\cdot; v):\ v \in \mathcal{V}_{\bar\nu} \big\},
\]
where
\begin{align*}
\mathcal{V}_{\bar\nu}
&:= \Big\{ v:\ 
\sup_{\varphi} |v_c(\varphi)| \le \bar\nu_c,\ \sup_\varphi \|v_u(\varphi)\|_2 \le \bar\nu_u,\ \\
&\hspace{3.0em} \sup_\varphi \|v_w(\varphi)\|_F \le \bar\nu_w \Big\}.
\end{align*}

\begin{remark} \label{rem:nu-bar}
Note that for any $\tilde f \in \mathcal{F}_{\bar\nu}$, we have
$\|\tilde{f}\|_{\mathcal{F}_{\bar{\nu}}} \leq \|\bar{\nu}\|_2$.
Thus, $\|\bar{\nu}\|_2$ quantifies the complexity of the function class $\mathcal{F}_{\bar\nu}$ and serves as a natural measure of the hardness of the regression problem. It will play a central role in controlling the approximation error (Lemma~\ref{lem:approximation}) and the convergence of \textsc{ProjGD} (Theorem~\ref{thm:projgd}).
\end{remark}

For any $\tilde{f}\in \mathcal{F}_{\bar{\nu}}$, the transportation mapping $v$ associated with $\tilde{f} $ ``transports'' the initial weights to a point at which the model $f$ more closely approximates $\tilde f$. Indeed, for \(i\in[m]\), define the following parameters
\begin{equation} \label{eq:good-params}
\begin{aligned}
\tilde c_i &:= c_i^{(0)} \;+\; m^{-1/2}\, v_c(\varphi_i^{(0)}),\\
\tilde U_i &:= U_i^{(0)} \;+\; c_i^{(0)} m^{-1/2}\, v_u(\varphi_i^{(0)}),\\
\tilde W_i &:= W_i^{(0)} \;+\; c_i^{(0)} m^{-1/2}\, v_w(\varphi_i^{(0)}),
\end{aligned}    
\end{equation}
and \(\tilde\varphi := (\tilde c_i,\tilde U_i,\operatorname{vec}(\tilde W_i))_{i=1}^m\). Then we have the following approximation error bound.

\begin{lemma}[Approximation error] \label{lem:approximation}
Let $\tilde{f}(\cdot; v) \in \mathcal{F}_{\bar{\nu}}$ and $\tilde{\varphi}$ be as in \eqref{eq:good-params}. Define $\varepsilon_{\mathrm{app}}:=\sup_{j\in[n]}\ \big|f_{\mathrm{lin}}(X^{(j)};\tilde\varphi) - \tilde f(X^{(j)};v)\big|$. Then for any $\delta \in (0,1)$, we have
\[
\varepsilon_{\mathrm{app}}
\leq 4(\sigma_0\bar{\nu}_c + \sigma_1\bar{\nu}_u+\sigma_1\bar{\nu}_w)\sqrt{\tfrac{\log(2n/\delta)}{m}} =: B_{\mathrm{app}, m}(\delta),
\]
with probability at least \(1-\delta\). We denote this event with $\mathcal E_{\mathrm{app}}(\delta)$.
\end{lemma}

\begin{remark}
In contrast to the RNN counterpart \citep[Proposition~3.5]{cayci2024rnngd}, where gradient norms may grow with the sequence length, the softmax-covariance structure of attention yields an approximation bound independent of $T$.
\end{remark}

\subsection{Main Results} \label{sec:main-results}
We now present the main theorem, which formally establishes the result stated in Section~\ref{sec:main-results-informal}.

\newcommand{\poly}{\mathrm{poly}}
\newcommand{\Lone}{L_{1,m}(\delta')}
\newcommand{\Ltwo}{L_{2,m}(\delta')}

\begin{theorem} [\textsc{ProjGD} Convergence] \label{thm:projgd}
Assume $f^\star \in \mathcal{F}_{\bar\nu}$ and the projection radius satisfies $\rho \ge \bar\nu$.
Run \textsc{ProjGD} onto $\Omega_\rho$ for $\tau$ steps with the step size $\eta=1/\sqrt{\tau}$.
Then, conditioned on the event $\mathcal{E}_{U,m}(\delta') \cap \mathcal E_{\mathrm{app}}(\delta)$, we have
\begin{multline*}
\min_{s<\tau}\ \widehat{\mathcal{L}}_n\!\big(\varphi^{(s)}\big)\; \lesssim_{\bar{\sigma}, \bar{\nu}, \rho}
\! \;
\frac{\Lone^{4}}{\sqrt{\tau}} \\
+\Lone(B_{\mathrm{app}, m}(\delta) + B_{\mathrm{lin}, m}(\delta') + \varepsilon_{\mathrm{CoF}}),
\end{multline*}

and for the average iterate $\bar\varphi^{(\tau)}:=\frac{1}{\tau}\sum_{s=0}^{\tau-1}\varphi^{(s)}$,
\begin{multline*}
\widehat{\mathcal L}_n\!\big(\bar\varphi^{(\tau)}\big) \; \lesssim_{\bar{\sigma}, \bar{\nu}, \rho}
\! \; \frac{\Lone^{4}}{\sqrt{\tau}} +\Lone\big(B_{\mathrm{app}, m}(\delta) \\
+ B_{\mathrm{lin}, m}(\delta') + \varepsilon_{\mathrm{CoF}}\big)
+B_{\mathrm{lin}, m}(\delta')^2,
\end{multline*}

where $B_{\mathrm{lin}, m}(\delta')$ and $B_{\mathrm{app}, m}(\delta)$ are defined as in Lemmas~\ref{lem:linearization} and \ref{lem:approximation}, and 
\[
\varepsilon_{\mathrm{CoF}} \; \lesssim_{\bar{\sigma}, \bar{\nu}, \rho} \; \frac{\Lone + \Ltwo}{\sqrt{m}}.
\]
\end{theorem}

\begin{remark}
The realizability condition $f^\star\in\mathcal{F}_{\bar\nu}$ is used only in the approximation step
(Lemma~\ref{lem:approximation}) to control how well $f^\star$ can be represented in the Transformer-NTK
RKHS. If $f^\star$ is not realizable by $\mathcal{F}_{\bar\nu}$, the optimization and linearization
arguments (Lemmas~\ref{lem:h-local-Lipschitz-smooth}--\ref{lem:linearization}) and the Lyapunov analysis in Theorem~\ref{thm:projgd} go through unchanged, and the final bound acquires the additional term
\[
\inf_{\tilde f\in\mathcal{F}_{\bar\nu}}\ \max_{j\in[n]}\big|f^\star(X^{(j)})-\tilde f(X^{(j)})\big|.
\]
We leave a detailed study of the approximation of broader target functions by the Transformer-NTK RKHS for future work.
\end{remark}

\begin{remark}
The condition $\rho\ge \bar\nu$ ensures
that the transported parameters $\tilde\varphi$ constructed from $v\in\mathcal{V}_{\bar\nu}$
(cf.\ \eqref{eq:good-params}) lie in the feasible set $\Omega_\rho$. Larger $\bar\nu$ corresponds to a richer target class (larger RKHS norm budget cf. Remark~\ref{rem:nu-bar}) and requires a larger projection radius, which in turn increases the hidden constants in
Theorem~\ref{thm:projgd}. See Appendix~\ref{ap:projgd} for explicit dependencies.
\end{remark}

We can state a similar result for the convergence of \textsc{ProjSGD} as well.

\begin{proposition} [\textsc{ProjSGD} Convergence] \label{thm:projsgd}
Running \textsc{ProjSGD} under the same setup as Theorem~\ref{thm:projgd} gives, on the event $\mathcal{E}_{U,m}(\delta') \cap \mathcal E_{\mathrm{app}}(\delta)$,
\begin{multline*}
\mathbb{E}\left[\min_{s<\tau}\widehat{\mathcal L}_n(\varphi^{(s)})\,|\,\varphi^{(0)}\right]\; \lesssim_{\bar{\sigma}, \bar{\nu}, \rho}
\! \;
\frac{\Lone^{4}}{\sqrt{\tau}} \\
+\Lone(B_{\mathrm{app}, m}(\delta) + B_{\mathrm{lin}, m}(\delta') + \varepsilon_{\mathrm{CoF}}),
\end{multline*}
and for the average iterate,
\begin{multline*}
\mathbb{E}\left[\widehat{\mathcal L}_n\!\big(\bar\varphi^{(\tau)}\big)\,|\,\varphi^{(0)}\right]
\! \; \lesssim_{\bar{\sigma}, \bar{\nu}, \rho}
\! \;
\frac{\Lone^{4}}{\sqrt{\tau}} \\
+\Lone(B_{\mathrm{app}, m}(\delta) + B_{\mathrm{lin}, m}(\delta') + \varepsilon_{\mathrm{CoF}})
+B_{\mathrm{lin}, m}(\delta')^2.
\end{multline*}
\end{proposition}

\paragraph{Comparison of results.}
Theorem~\ref{thm:projgd} gives a high-probability finite-time bound for \textsc{ProjGD}, stated for the minimum iterate and the averaged iterate. Proposition~\ref{thm:projsgd}
shows that \textsc{ProjSGD} enjoys the \emph{same} error decomposition and the \emph{same} rates in $(\tau,m,n)$, with the
guarantee stated in \emph{conditional expectation} over the sampling randomness.

\section{EXPERIMENTAL RESULTS} \label{sec:exp}
We complement the theory with two experiments. Section~\ref{sec:teacher-student} validates the scaling laws predicted by Lemma~\ref{lem:linearization}, Lemma~\ref{lem:approximation}, and Theorem~\ref{thm:projgd} in a realizable teacher--student setting aligned with $f^\star\in\mathcal{F}_{\bar\nu}$ and near-initialization assumptions. Section~\ref{sec:rnn-vs-transformer} then probes behavior beyond this regime by comparing IndRNNs and our Transformer on an autoregressive AR($L$) task trained without projection, highlighting long-term dependency behavior and the associated memory trade-off.

\subsection{Teacher--Student Scaling-Law Validation} \label{sec:teacher-student}

We assess the three components of our theory---linearization, approximation, and optimization---with controlled teacher--student experiments using the same shallow Transformer as in Section~\ref{sec:arch} with $q_X=X_T$. Throughout, inputs are i.i.d. $X_t \sim \mathcal{N}(0, I_d)$ with token-wise clipping to keep $\max_t \|X_t\|_2\le1$.

\paragraph{Teacher.}
To realize a ground-truth $f^\star\in\mathcal{F}_{\bar\nu}$ as in \eqref{eq:f_tilde_def} we define a sup-norm-bounded transportation mapping $v \in \mathcal{V}_{\bar{\nu}}$ as \(v(\varphi)=\frac{1}{R}\sum_{r=1}^R \, v^{(r)}(\varphi)\), where
\[
\begin{aligned}
v_c^{(r)}(\varphi)&=\tfrac{\nu_c}{\sigma_0}\,\phi_c(X_R^{(r)}),\quad
v_u^{(r)}(\varphi)=\tfrac{\nu_u}{\sigma_1}\,\phi_u(X_R^{(r)}),\\
v_w^{(r)}(\varphi)&=\operatorname{clip}_F\!\big(\phi_w(X_R^{(r)}),\,\nu_w\big),
\end{aligned}
\]
and the anchor samples \(X_R^{(r)}\) are drawn from the same distribution as inputs once and kept fixed. For a large $\tilde{m} \in \mathbb{N}$, we sample an i.i.d.\ pool $\{\varphi_i^{(0)}\}_{i=1}^{\tilde m}$ and generate the labels via a Monte Carlo approximation of~\eqref{eq:f_tilde_def}, i.e., 
\begin{multline*}
    \tilde{y}^{(j)} = \frac{1}{\tilde m}\sum_{i=1}^{\tilde m}\!\big[\phi_c(X^{(j)})\,v_c(\varphi_i^{(0)}) \\+\langle \phi_u(X^{(j)}),v_u(\varphi_i^{(0)})\rangle +\langle \phi_w(X^{(j)}),v_w(\varphi_i^{(0)})\rangle_F\big],
\end{multline*}
so the teacher lies in the RKHS induced by the model’s NTK, matching the realizability assumption used in the proofs. All experiments are conducted with $\tilde{m}=8192$, $d=8$, $T=16$, $n=5000$, $\bar\nu=(3,3,3)$, and $\tanh$ activation.

\paragraph{Students and metrics.}
Students are initialized with the symmetric initialization from Section~\ref{sec:init} so that $f(\cdot;\varphi^{(0)})\equiv0$. For each width $m \in \{8, 16, 32, 64, 128, 256\}$, we record:

\textbf{\textbullet\hskip 6pt Linearization error.} We select parameters on the boundary of the neighborhood $\Omega_\rho$ and compute $\sup_{j\in[n]} \big| f(X^{(j)};\varphi) - f_{\mathrm{lin}}(X^{(j)};\varphi) \big|$, which should scale like $m^{-1/2}$ per Lemma~\ref{lem:linearization}.

\textbf{\textbullet\hskip 6pt Approximation error.} We set the parameters $\tilde\varphi$ as in \eqref{eq:good-params} and measure $\sup_{j\in[n]}\ \big| f_{\mathrm{lin}}(X^{(j)};\tilde\varphi) - \tilde{y}^{(j)} \big|$, which is expected to scale like $m^{-1/2}$ from Lemma~\ref{lem:approximation}.

\textbf{\textbullet\hskip 6pt Minimum training loss.} We train a student Transformer with $\tau=4000$ steps, a constant step size $\eta=1/\sqrt{\tau}$, and projection radii $\rho=\bar\nu$. Minimum training loss $\min_{t<\tau}\widehat{\mathcal{L}}_n(\varphi^{(t)})$ versus $m$ is supposed to scale like $m^{-1/2}$ from Theorem~\ref{thm:projgd}.

\begin{figure*}[!t]
  \centering
  \begin{subfigure}[t]{0.325\textwidth}
    \centering
    \includegraphics[width=\linewidth]{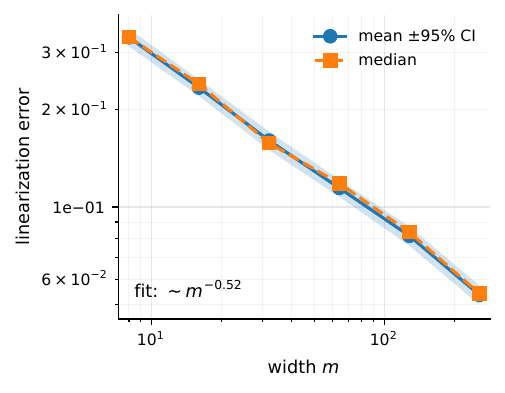}
    \caption{Linearization error}\label{fig:linearization}
  \end{subfigure}\hfill
  \begin{subfigure}[t]{0.325\textwidth}
    \centering
    \includegraphics[width=\linewidth]{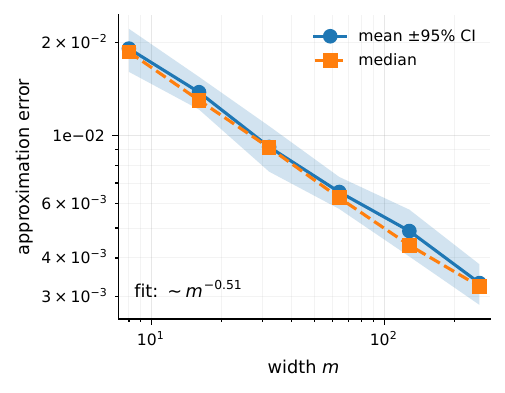}
    \caption{Approximation error}\label{fig:approximation}
  \end{subfigure}\hfill
  \begin{subfigure}[t]{0.325\textwidth}
    \centering
    \includegraphics[width=\linewidth]{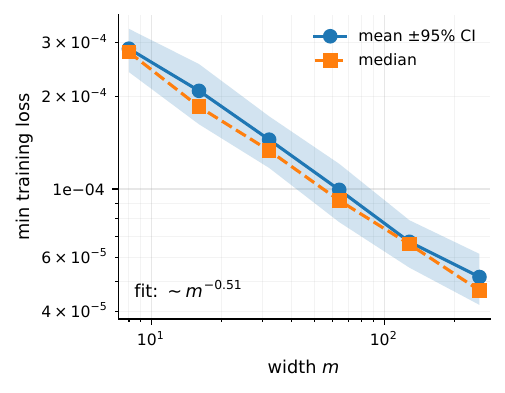}
    \caption{Min training loss}\label{fig:min-train-loss}
  \end{subfigure}
  \caption{Width scaling across three metrics (log–log). Shaded bands show mean $\pm$95\% CI across seeds. Fitted slopes are consistent with the predicted $m^{-1/2}$ decay.}
  \label{fig:teacher-student}
\end{figure*}

We average curves over $10$ seeds, show $\pm$95\% confidence interval (CI) around the mean, and fit a straight line to the log–log curves to estimate the scaling exponent. Figure~\ref{fig:teacher-student} shows that all three quantities—linearization error, approximation error, and minimum training loss—decrease monotonically with width and exhibit a clear power-law dependence on $m$. The fitted slopes are in close agreement with the predicted $m^{-1/2}$ rate. Medians closely track the mean and the CIs remain small across widths, indicating that the observed scaling is robust to random initialization.

\subsection{IndRNN vs. Transformer Architectures} \label{sec:rnn-vs-transformer}
To compare how IndRNNs~\citep{cayci2024rnngd} and our Transformer cope with long-term dependencies and training stability, we use a controlled autoregressive AR($L$) task where the lag $L$ sets the dependency horizon. Theory for diagonal RNNs predicts that, depending on a Lipschitz cutoff, gradients can grow exponentially due to repeated multiplication over a horizon of $T$ steps~\citep{cayci2024rnngd}, whereas our Transformer's bound has no explicit $T$-dependence (cf.\ Lemma~\ref{lem:h-local-Lipschitz-smooth}).

Each input sequence $X\!\in\!\mathbb{R}^{1\times T}$ is constructed as a window from a univariate process
\[
X_t=\alpha\,X_{t-L}+\varepsilon_t,\qquad \varepsilon_t\sim\mathcal{N}(0,\sigma_\varepsilon^2),
\]
and the label is $y=X_{T+1}$. We append fixed positional channels ($d_{\text{pos}}{=}8$ sinusoidal pairs) to the input so that the Transformer has positional information. We fix $n=5{,}000$, $\alpha=0.9$, $\sigma_\varepsilon^2=0.1$, and vary $L$. For the Transformer we set $T=L+1$ so that the lag falls within the context window, yielding a memory footprint $\mathcal{O}(L)$ compared to IndRNN's $\mathcal{O}(1)$. Both models use symmetric initialization and are trained with gradient descent for $\tau=2{,}000$ steps with step size $\eta=1/\sqrt{\tau}$ and no projection to observe gradient-norm dynamics. We set $m=64$ and use $\tanh$ as the activation. Results are averaged over $20$ seeds.

\paragraph{Metrics.}
We track training stability via the maximum Jacobian norm over training, computed with respect to the diagonal recurrent weights (IndRNN) or the attention weights (Transformer). Predictive performance is summarized by the minimum validation loss achieved over training.

Figure~\ref{fig:rnn-vs-trf} summarizes the comparison across lags $L$. The top panel shows that IndRNN exhibits larger and more variable norms, consistent with the potential gradient growth predicted in \citet{cayci2024rnngd}, whereas the Transformer remains stable. The bottom panel shows the best validation loss achieved over training: IndRNN performs well for small lags but degrades as $L$ increases, while the Transformer maintains roughly constant performance as $L$ grows, at the cost of an $\mathcal{O}(L)$ context window. Overall, the experiment illustrates the performance-memory trade-off between the two architectures anticipated by the theory. Although this task lies outside our realizable/projection regime, the observed stability of attention is consistent with the softmax-covariance control highlighted in Remark~\ref{rem:cov-bound}.

\begin{figure}[t]
  \centering
  \includegraphics[width=\linewidth]{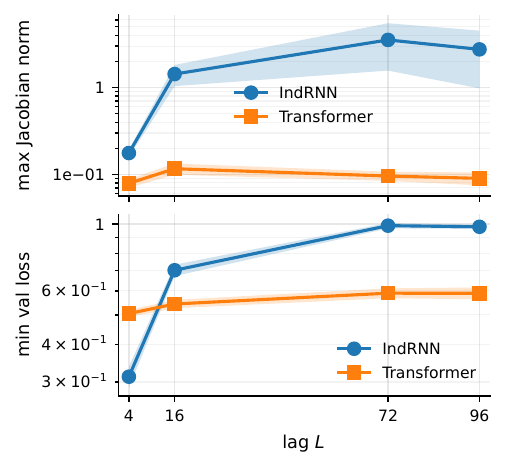}
  \caption{IndRNN vs.\ Transformer on the AR($L$) task with $T=L+1$. Top: maximum Jacobian norm over training (stability); bottom: minimum validation loss achieved over training (performance). Results are averaged over $20$ seeds.}
  \label{fig:rnn-vs-trf}
\end{figure}

\section{CONCLUSION}
We performed a finite-time NTK analysis of projected gradient methods for multi-head shallow Transformers. We established nonasymptotic convergence bounds with explicit dependencies on the training period, width, and the number of samples. To the best of our knowledge, this is the first work that establishes near-optimality of first-order methods without massive overparameterization for Transformers. Interesting future directions that can build on this work include the study of deep Transformer architectures and the extension of the setting to sequence-to-sequence learning.  

\subsubsection*{Acknowledgements}
We thank the anonymous reviewers for their thoughtful comments and helpful suggestions, which improved the clarity and presentation of the paper. This work was supported in part by the NSF grants: NSF AI Institute (AI-EDGE) 2112471, CNS-NeTS-2106679, the ONR Grant N00014-19-1-2621; and the ARO Grant W911NF-24-1-0103.

\bibliography{references}

\pagebreak
\section*{Checklist}



\begin{enumerate}

  \item For all models and algorithms presented, check if you include:
  \begin{enumerate}
    \item A clear description of the mathematical setting, assumptions, algorithm, and/or model. [Yes] \textbf{Justification:} The setup, the model, and the algorithm are described in Section~\ref{sec:setup}. The function that is assumed to generate the underlying data is described in Section~\ref{sec:transportation-maps}.
    \item An analysis of the properties and complexity (time, space, sample size) of any algorithm. [Yes] \textbf{Justification:} The analysis of the algorithms is informally made in Section~\ref{sec:main-results-informal} and formally made in Section~\ref{sec:main-results}.
    \item (Optional) Anonymized source code, with specification of all dependencies, including external libraries. [Yes] \textbf{Justification:} The source code is provided at \url{https://github.com/enesarda22/nonasymptotic-transformer}.
  \end{enumerate}

  \item For any theoretical claim, check if you include:
  \begin{enumerate}
    \item Statements of the full set of assumptions of all theoretical results. [Yes] \textbf{Justification:} Every stated theoretical result contains the assumptions if any.
    \item Complete proofs of all theoretical results. [Yes] \textbf{Justification:} Complete proofs are provided in the Supplementary Material.
    \item Clear explanations of any assumptions. [Yes] \textbf{Justification:} The only assumption we make is on the function generating the underlying data, which is explained in Section~\ref{sec:transportation-maps}.
  \end{enumerate}

  \item For all figures and tables that present empirical results, check if you include:
  \begin{enumerate}
    \item The code, data, and instructions needed to reproduce the main experimental results (either in the supplemental material or as a URL). [Yes] \textbf{Justification:} The code and instructions are provided in a Github repo in the Supplementary Material.
    \item All the training details (e.g., data splits, hyperparameters, how they were chosen). [Yes] \textbf{Justification:} Experimental details are provided in Section~\ref{sec:exp} and additional details are provided in the Supplementary Material.
    \item A clear definition of the specific measure or statistics and error bars (e.g., with respect to the random seed after running experiments multiple times). [Yes] \textbf{Justification:} Experimental details are provided in Section~\ref{sec:exp} and additional details are provided in the Supplementary Material.
    \item A description of the computing infrastructure used. (e.g., type of GPUs, internal cluster, or cloud provider). [No]
    \textbf{Justification:} Conducted experiments are not resource-heavy, so we do not describe the computing infrastructure.

  \end{enumerate}

  \item If you are using existing assets (e.g., code, data, models) or curating/releasing new assets, check if you include:
  \begin{enumerate}
    \item Citations of the creator If your work uses existing assets. [Not Applicable]
    \item The license information of the assets, if applicable. [Not Applicable]
    \item New assets either in the supplemental material or as a URL, if applicable. [Not Applicable]
    \item Information about consent from data providers/curators. [Not Applicable]
    \item Discussion of sensible content if applicable, e.g., personally identifiable information or offensive content. [Not Applicable]
  \end{enumerate}

  \item If you used crowdsourcing or conducted research with human subjects, check if you include:
  \begin{enumerate}
    \item The full text of instructions given to participants and screenshots. [Not Applicable]
    \item Descriptions of potential participant risks, with links to Institutional Review Board (IRB) approvals if applicable. [Not Applicable]
    \item The estimated hourly wage paid to participants and the total amount spent on participant compensation. [Not Applicable]
  \end{enumerate}

\end{enumerate}

\clearpage
\appendix
\thispagestyle{empty}

\setcounter{lemma}{0}
\setcounter{theorem}{0}
\setcounter{proposition}{0}

\onecolumn
\aistatstitle{Supplementary Material}

\paragraph{Organization.}
In Appendix~\ref{sec:app-missing-proofs}, we restate and prove the results in the main text, and in Appendix~\ref{sec:app-exp-details}, we give further experimental details along with the code to reproduce the experiments.

\section{PROOFS} \label{sec:app-missing-proofs}

\subsection{Useful facts}\label{sec:useful-facts}

\begin{fact}[Jacobian of softmax] \label{fact:softmax-jacob}
For \(\sigma_s:\mathbb{R}^T\to\Delta_T\) given by
\((\sigma_s(z))_t=\exp(z_t)/\sum_{s=1}^T\exp(z_s)\),
the Jacobian equals
\[
J_s(z)\;=\;\operatorname{diag}\!\big(\sigma_s(z)\big)\;-\;\sigma_s(z)\,\sigma_s(z)^\top.
\]
It is symmetric, satisfies \(J_s(z) \mathbf{1}=0\), and obeys the operator bound
\(\|J_s(z)\|_{\infty\to 1}\le 2\). Here, $\mathbf{1}$ is the all‑ones vector.
\end{fact}
\begin{proof}
\(\partial(\sigma_s)_i/\partial z_j=\sigma_{s,i}(\delta_{ij}-\sigma_{s,j})\) gives the formula and symmetry.
Row sums vanish since \(\sum_j \partial(\sigma_s)_i/\partial z_j=0\), hence \(J_s(z)\mathbf{1}=0\).
For the \(\infty\!\to\!1\) bound, if \(\|u\|_\infty\le 1\) then
\(\|J_s(z) u\|_1=\|\operatorname{diag}(\alpha)u-\alpha(\alpha^\top u)\|_1
\le \sum_i \alpha_i|u_i|+|\alpha^\top u|\,\|\alpha\|_1\le 1+1=2\) with \(\alpha=\sigma_s(z)\).
\end{proof}

\begin{fact}[Bound for \(a(X;W)\)] \label{fact:a-bound}
Fix $X \in \mathcal{X}$. With \(\alpha=\sigma_s(X^\top W q_X)\in\Delta_T\),
\[
a(X;W)\;=\;X \alpha\;=\;\sum_{t=1}^T \alpha_t\,X_t \quad\Rightarrow\quad
\|a(X;W)\|_2\;\le\;\sum_{t=1}^T \alpha_t\|X_t\|_2\;\le\; \sum_{t=1}^T \alpha_t \;=\;1.
\]
\end{fact}

\begin{fact}[Bound for \(M(X;W)\)] \label{fact:cov-bound}
Fix $X \in \mathcal{X}$. Let \(\alpha=\sigma_s(X^\top W q_X)\) and \(\mu=X\alpha\).
Then
\[
M(X;W) \;=\; X\big(\operatorname{diag}(\alpha)-\alpha \alpha^\top\big)X^\top
\;=\; \sum_{t=1}^T \alpha_t\,(X_t-\mu)(X_t-\mu)^\top
\;\succeq\; 0.
\]
Consequently,
\[
M \;\preceq\; \sum_{t=1}^T \alpha_t\,X_t X_t^\top
\quad\text{and}\quad
\|M\|_F
\;\le\; \mathrm{tr}(M)
\;\le\; \sum_{t=1}^T \alpha_t \|X_t\|_2^2
\;\le\; 1.
\]
\end{fact}

\begin{fact}[Gaussian norm concentration]
If \(G\sim\mathcal{N}(0,I_d)\), then
\[
\mathbb{P}\!\left(\big|\|G\|_2-\mathbb{E}\|G\|_2\big|\ge t\right)\;\le\;2e^{-t^2/2}
\quad\text{for all }t>0,
\qquad
\mathbb{E}\|G\|_2 \;\le\; \sqrt{d}.
\]
In particular, for independent \(G_1,\dots,G_{m}\sim\mathcal{N}(0,I_d)\),
\[
\mathbb{P}\!\left(\max_{1\le i\le m}\|G_i\|_2
\;\le\; \sqrt{d}+\sqrt{2\log(m/\delta)}\right)\;\ge\; 1-\delta.
\]    
\end{fact}
\begin{proof}
By Lipschitz Gaussian concentration \citep[Theorem~5.6]{boucheron2013concentration} applied to \(u \mapsto \|u\|_2\),
we get the tail bound; the expectation bound follows from
\(\mathbb{E}\|G\|_2\le\sqrt{\mathbb{E}\|G\|_2^2}=\sqrt{d}\).
The union bound over \(m\) copies yields the shown high‑probability bound.
\end{proof}

\subsection{Proof of Lemma~\ref{lem:grads}}

\begin{lemma}
    Fix \(i\in[m]\) and write \(a_i:=a(X;W_i)\).
    Then
    \begin{align}
    \frac{\partial f(X;\varphi)}{\partial c_i}
    &= m^{-1/2}\, h(X;\theta_i),\\
    \nabla_{U_i} f(X;\varphi)
    &= m^{-1/2}\, c_i\, \sigma'\!\big(U_i^\top a_i\big)\, a_i,\\
    \nabla_{W_i} f(X;\varphi)
    &= m^{-1/2}\, c_i\, \sigma'\!\big(U_i^\top a_i\big)\, \big(M_i U_i\big)\, q_X^\top, 
    \end{align}
    where
    \begin{equation}
    M_i \;:=\; M(X;W_i) \;=\; X\, J_s\!\big(X^\top W_i q_X\big)\, X^\top,
    \end{equation}
    and
       $ J_s(z) \;=\; \operatorname{diag}\!\big(\sigma_s(z)\big) - \sigma_s(z)\sigma_s(z)^\top$
    is the Jacobian matrix of softmax.
\end{lemma}

\begin{proof}
Write \(\alpha_i \coloneqq \sigma_s\!\big(X^\top W_i q_X\big)\in\Delta_T\),
\(a_i = X\alpha_i\in\mathbb{R}^d\), and \(z_i \coloneqq X^\top W_i q_X\in\mathbb{R}^T\).

\paragraph{Derivative w.r.t.\ \(c_i\).}
From \(f(X;\varphi) = \tfrac{1}{\sqrt{m}}\sum_{j=1}^m c_j h(X;\theta_j)\),
\(\partial f/\partial c_i = m^{-1/2} \, h(X;\theta_i)\).

\paragraph{Gradient w.r.t.\ \(U_i\).}
Since \(h(X;\theta_i) = \sigma(U_i^\top a_i)\),
\[
\nabla_{U_i} h(X;\theta_i) \;=\; \sigma'\!\big(U_i^\top a_i\big)\, a_i,
\qquad
\nabla_{U_i} f \;=\; m^{-1/2}\, c_i\, \nabla_{U_i} h
\;=\; m^{-1/2}\, c_i\, \sigma'\!\big(U_i^\top a_i\big)\, a_i.
\]

\paragraph{Gradient w.r.t.\ \(W_i\).}
We compute via differentials. First, by the chain rule for softmax,
\[
d\alpha_i \;=\; J_s(z_i)\, dz_i,
\qquad
dz_i \;=\; X^\top\, dW_i\, q_X,
\]
and hence
\[
da_i \;=\; X\, d\alpha_i \;=\; X\, J_s(z_i)\, X^\top\, dW_i\, q_X \;=\; M_i\, dW_i\, q_X.
\]
Next,
\[
dh(X;\theta_i) \;=\; \sigma'\!\big(U_i^\top a_i\big)\, U_i^\top\, da_i
\;=\; \sigma'\!\big(U_i^\top a_i\big)\, U_i^\top\, M_i\, dW_i\, q_X.
\]
Therefore
\[
df \;=\; m^{-1/2}\, c_i\, dh
\;=\; m^{-1/2}\, c_i\, \sigma'\!\big(U_i^\top a_i\big)\, U_i^\top M_i\, dW_i\, q_X.
\]
Using \(\operatorname{tr}(A^\top B)=\langle A,B\rangle_F\) and cyclicity of trace,
\[
U_i^\top M_i\, dW_i\, q_X
\;=\; \operatorname{tr}\!\big(q_X\, U_i^\top M_i\, dW_i\big)
\;=\; \operatorname{tr}\!\big((M_i U_i q_X^\top)^\top dW_i\big),
\]
where we used that \(M_i\) is symmetric. Identifying the coefficient of
\(dW_i\) in the Frobenius inner product yields
\[
\nabla_{W_i} f(X;\varphi)
\;=\; m^{-1/2}\, c_i\, \sigma'\!\big(U_i^\top a_i\big)\, (M_i U_i)\, q_X^\top,
\]
as claimed.
\end{proof}
\subsection{Proof of Lemma~\ref{lem:h-local-Lipschitz-smooth}}
\begin{lemma}
    On $\mathcal E_{U,m}(\delta')$, the mapping $\theta_i \mapsto h(X;\theta_i)$ is $L_{1,m}(\delta')$-Lipschitz and $L_{2,m}(\delta')$-smooth in $\Omega_\rho$, where
    \[
        L_{1,m}(\delta') := \sigma_1\sqrt{1 + B_{U,m}(\delta')^2},
    \]
    and
    \[
        L_{2,m}(\delta') := \sigma_2(1 + B_{U,m}(\delta')^2) + 8\sigma_1\sqrt{1 + B_{U,m}(\delta')^2},
    \]
where $B_{U,m}(\delta')$ is defined as in \eqref{eq:U-norm-bound}.
\end{lemma}

\begin{proof}
Throughout, $D[\cdot]$ denotes the Fr\'echet derivative. For a scalar map $g:\mathbb{R}^{d\times d} \to \mathbb{R}$ the Fréchet derivative is $Dg(W)[\Delta]=\left.\frac{d}{dt}g(W+t\Delta)\right|_{t=0}$.
By the Riesz representation theorem, there is a unique $\nabla_W g(W)$ with 
$Dg(W)[\Delta]=\langle\nabla_W g(W),\Delta\rangle_F$, hence
\(
\|\nabla_W g(W)\|_F=\sup_{\|\Delta\|_F=1}|Dg(W)[\Delta]|.
\)
For a matrix-valued operator $G$, we measure $DG(W)$ by the induced norm 
$\|DG(W)\|_{F\to F}:=\sup_{\|\Delta\|_F=1}\|DG(W)[\Delta]\|_F$.

\medskip
Fix \(i\in[m]\) and write
\[
\alpha_i \;=\; \sigma_s\!\big(X^\top W_i q_X\big)\in\Delta_T, \qquad
a_i \;=\; a(X;W_i) \;=\; X \alpha_i,\qquad
z_i \;=\; U_i^\top a_i,
\]
and
\[
M_i \;=\; X\, J_s\!\big(X^\top W_i q_X\big)\, X^\top,
\]
where $J_s$ is the softmax Jacobian.

\paragraph{(a) Local Lipschitz continuity.}
From Lemma~\ref{lem:grads},
\[
\nabla_{U_i} h(X;\theta_i) = \sigma'(z_i)\,a_i,
\qquad
\nabla_{W_i} h(X;\theta_i) = \sigma'(z_i)\, (M_i U_i)\, q_X^\top .
\]
Therefore
\[
\big\|\nabla_{U_i} h(X;\theta_i)\big\|_2 \le \sigma_1 \|a_i\|_2 \le \sigma_1,
\qquad
\big\|\nabla_{W_i} h(X;\theta_i)\big\|_F
\le \sigma_1 \|M_i\|_F\|U_i\|_2\|q_X\|_2 \le \sigma_1 \|U_i\|_2.
\]
Collecting the two blocks,
\[
\big\|\nabla_{\theta_i} h(X;\theta_i)\big\|_2^2
= \big\|\nabla_{U_i} h\big\|_2^2 + \big\|\nabla_{W_i} h\big\|_F^2
\;\le\; \sigma_1^2\,(1+\|U_i\|_2^2).
\]
Hence, on \(\mathcal E_{U,m}(\delta')\), the mapping \( \theta_i \mapsto h(X;\theta_i)\) is \(L_{1,m}(\delta')\)-Lipschitz in \(\Omega_\rho\) with
\(
L_{1,m}(\delta') = \sigma_1\sqrt{1+B_{U,m}(\delta')^2}.
\)

\paragraph{(b) Local smoothness.}
Write \(h_i := h(X;\theta_i) = \sigma(z_i)\).
By the chain rule,
\[
\nabla_{\theta_i} h_i = \sigma'(z_i)\,\nabla_{\theta_i} z_i,
\qquad
\nabla_{\theta_i}^2 h_i = \sigma''(z_i)\,(\nabla_{\theta_i} z_i)(\nabla_{\theta_i} z_i)^\top
+ \sigma'(z_i)\,\nabla_{\theta_i}^2 z_i,
\]
and thus
\begin{equation}\label{eq:hess-bound-template}
\big\|\nabla_{\theta_i}^2 h_i\big\|_{\mathrm{op}}
\;\le\; \sigma_2 \underbrace{\big\|\nabla_{\theta_i} z_i\big\|_2^2}_{(i)}
\;+\; \sigma_1 \underbrace{\big\|\nabla_{\theta_i}^2 z_i\big\|_{\mathrm{op}}}_{(ii)}.
\end{equation}

\emph{Bounding \((i)\).}
We have \(\nabla_{U_i} z_i = a_i\) and, for any direction \(\Delta\in\mathbb{R}^{d\times d}\),
\[
D z_i[\Delta] \;=\; U_i^\top M_i\, \Delta\, q_X,
\quad\Rightarrow\quad
\|\nabla_{W_i} z_i\|_F \le \|M_i\|_{F}\, \|U_i\|_2\, \|q_X\|_2 \le \|U_i\|_2.
\]
Therefore
\begin{equation}\label{eq:gradzi-bound}
\big\|\nabla_{\theta_i} z_i\big\|_2^2
= \|\nabla_{U_i} z_i\|_2^2 + \|\nabla_{W_i} z_i\|_F^2
\;\le\; 1 + \|U_i\|_2^2.
\end{equation}

\emph{Bounding \((ii)\).}
Note \(\nabla_{U_i}^2 z_i = 0\) and the $W_i$–$U_i$ block of the Hessian equals the Jacobian of \(a_i\) w.r.t.~\(W_i\):
\[
D\!\left[\nabla_{U_i} z_i\right][\Delta] \;=\; D[a_i][\Delta] \;=\; M_i\, \Delta\, q_X,
\quad\Rightarrow\quad
\big\|\nabla^2_{U_i,\operatorname{vec}(W_i)} z_i\big\|_{\mathrm{op}} \leq \big\|\nabla^2_{U_i,\operatorname{vec}(W_i)} z_i\big\|_F \;\le\; \|M_i\|_{F}\,\|q_X\|_2 \;\le\; 1.
\]
For the \(W_i\)–\(W_i\) block, differentiate \(\nabla_{W_i} z_i = M_i U_i q_X^\top\) once more in the direction \(\Delta\):
\[
D\!\left[\nabla_{W_i} z_i\right][\Delta]
\;=\; D[M_i][\Delta]\, U_i\, q_X^\top.
\]
From the covariance form of \(M_i\),
\(
D[M_i][\Delta] = \sum_{t=1}^T (D[\alpha_i][\Delta])_t\, (X_t-\mu_i)(X_t-\mu_i)^\top,
\)
with \(\mu_i := \sum_{t=1}^T (\alpha_i)_t X_t = a_i\) and
\(D[\alpha_i][\Delta] = J_s(X^\top W_i q_X)\, X^\top \Delta\, q_X\).
Hence
\[
\|D[M_i][\Delta]\|_F \;\le\; \sum_{t=1}^T |(D \alpha_i[\Delta])_t|\, \|X_t-\mu_i\|_2^2
\;\le\; 4\,\|D \alpha_i[\Delta]\|_1,
\]
and using \(\| \operatorname{diag}(p)-pp^\top \|_{\,\infty\to 1}\le 2\) (Fact~\ref{fact:softmax-jacob}),
\[
\|D \alpha_i[\Delta]\|_1
= \big\|(\operatorname{diag}(\alpha_i)-\alpha_i \alpha_i^\top)\, X^\top \Delta\, q_X\big\|_1
\;\le\; 2\,\|X^\top \Delta\, q_X\|_\infty
\;\le\; 2\,\|\Delta\|_F,
\]
where we used \(|X_t^\top \Delta q_X|\le \|X_t\|_2\,\|\Delta\|_F\,\|q_X\|_2\le \|\Delta\|_F\).
Thus \(\|D[M_i][\Delta]\|_F \le 8 \|\Delta\|_F\), and therefore
\[
\|D\!\left[\nabla_{W_i} z_i\right][\Delta]\|_F
\leq \|D[M_i][\Delta]\|_F \|U_i\|_2 \|q_X\|_2
\leq 8 \|\Delta\|_F \|U_i\|_2,
\]
which implies 
\[
\big\|\nabla^2_{\operatorname{vec}(W_i)} z_i\big\|_{\mathrm{op}} \;\le\; 8\,\|U_i\|_2.
\]
Let $B_i := \nabla^2_{U_i,\operatorname{vec}(W_i)} z_i$ and $C_i := \nabla^2_{\operatorname{vec}(W_i)} z_i.$
Then
\[
\nabla_{\theta_i}^2 z_i
=
\begin{pmatrix}
0 & B_i^\top\\
B_i & C_i
\end{pmatrix}.
\]
For any \((u,v)\) with \(\|u\|_2^2+\|v\|_2^2=1\),
\[
\left\|
\begin{pmatrix}
0 & B_i^\top\\
B_i & C_i
\end{pmatrix}
\binom{u}{v}
\right\|_2^2
\le
\|B_i\|_{\mathrm{op}}^2 \|v\|_2^2
+
\big(\|B_i\|_{\mathrm{op}}\|u\|_2+\|C_i\|_{\mathrm{op}}\|v\|_2\big)^2
\le
2\|B_i\|_{\mathrm{op}}^2+\|C_i\|_{\mathrm{op}}^2.
\]
Therefore
\begin{equation}\label{eq:hesszi-bound}
\big\|\nabla_{\theta_i}^2 z_i\big\|_{\mathrm{op}}
\;\le\;
\sqrt{2\big\|\nabla^2_{U_i,\operatorname{vec}(W_i)} z_i\big\|_{\mathrm{op}}^2
      + \big\|\nabla^2_{\operatorname{vec}(W_i)} z_i\big\|_{\mathrm{op}}^2}
\;\le\; \sqrt{2 + 64\,\|U_i\|_2^2}
\;\le\; 8\sqrt{1+\|U_i\|_2^2}.
\end{equation}

Finally, plugging \eqref{eq:gradzi-bound} and \eqref{eq:hesszi-bound} into \eqref{eq:hess-bound-template} yields
\[
\big\|\nabla_{\theta_i}^2 h(X;\theta_i)\big\|_{\mathrm{op}}
\;\le\; \sigma_2\,(1+\|U_i\|_2^2) + 8\sigma_1\,\sqrt{1+\|U_i\|_2^2}.
\]
Therefore, on \(\mathcal E_{U,m}(\delta')\) and for all \(U_i\in\mathcal{U}_{\rho_u,i}\),
\[
\big\|\nabla_{\theta_i} h(X;\theta_i)\big\|_2 \le \sigma_1\sqrt{1+B_{U,m}(\delta')^2}
\quad\text{and}\quad
\big\|\nabla_{\theta_i}^2 h(X;\theta_i)\big\|_{\mathrm{op}}
\le \sigma_2\!\big(1+B_{U,m}(\delta')^2\big)
   + 8\sigma_1\sqrt{1+B_{U,m}(\delta')^2}.
\]
\end{proof}

\subsection{Proof of Lemma~\ref{lem:linearization}}
\begin{lemma}
Assume \(\varphi \in \Omega_\rho\) and define $\varepsilon_{\mathrm{lin}}:=\sup_{j\in[n]} \big| f(X^{(j)};\varphi) - f_{\mathrm{lin}}(X^{(j)};\varphi) \big|$. Conditioned on the event $\mathcal{E}_{U,m}(\delta')$, we have
\begin{equation*}
\varepsilon_{\mathrm{lin}}
\leq \tfrac{1}{\sqrt{m}}\Bigl(L_{1,m}(\delta')\rho_c\sqrt{\rho_w^2+\rho_u^2}
+ L_{2,m}(\delta')(\rho_w^2 + \rho_u^2)\Bigr),
\end{equation*}
where the upper bound is denoted as $B_{\mathrm{lin},m}(\delta').$
\end{lemma}

\begin{proof}
Fix any $j\in[n]$ and write $X=X^{(j)}$.
For $i\in[m]$ set $\delta c_i := c_i-c_i^{(0)}$, $\delta U_i := U_i-U_i^{(0)}$, $\delta W_i := W_i-W_i^{(0)}$, and
\[
\delta \theta_i \;:=\; (\delta U_i,\operatorname{vec}(\delta W_i)), 
\qquad 
\|\delta\theta_i\|_2^2 \;=\; \|\delta U_i\|_2^2 + \|\delta W_i\|_F^2 .
\]
We can write $f_{\mathrm{lin}}(X;\varphi)$ as
\begin{align*}
f_{\mathrm{lin}}(X;\varphi)
&= f\big(X;\varphi^{(0)}\big)
   \;+\; \sum_{i=1}^m \Big\langle \nabla_{c_i} f\big(X;\varphi^{(0)}\big),\, \delta c_i \Big\rangle
   \;+\; \sum_{i=1}^m \Big\langle \nabla_{\theta_i} f\big(X;\varphi^{(0)}\big),\, \delta\theta_i \Big\rangle\\
&= f\big(X;\varphi^{(0)}\big)
   \;+\; \frac{1}{\sqrt{m}} \sum_{i=1}^m
   \Big[\, h\!\big(X;\theta_i^{(0)}\big)\, \delta c_i
         + c_i^{(0)} \,\big\langle \nabla_{\theta_i} h\!\big(X;\theta_i^{(0)}\big),\, \delta\theta_i \big\rangle \Big].
\end{align*}
A direct expansion gives the per-sample linearization error
\begin{equation}\label{eq:elin-decomp}
\begin{aligned}
e_{\mathrm{lin}}(X)
&:= f(X;\varphi)-f_{\mathrm{lin}}(X;\varphi)\\
&= \frac{1}{\sqrt{m}} \sum_{i=1}^m \delta c_i\,\Big( h(X;\theta_i)-h\!\big(X;\theta_i^{(0)}\big) \Big)\\
&\quad + \frac{1}{\sqrt{m}} \sum_{i=1}^m c_i^{(0)}
\Big( h(X;\theta_i)-h\!\big(X;\theta_i^{(0)}\big)
      - \big\langle \nabla_{\theta_i} h\!\big(X;\theta_i^{(0)}\big),\, \delta\theta_i \big\rangle \Big).
\end{aligned}
\end{equation}
where we use the fact that $f(X;\varphi^{(0)})=0$. Since $\varphi\in\Omega_\rho$, we have
\[
|\delta c_i|\le \frac{\rho_c}{\sqrt{m}},
\qquad
\|\delta\theta_i\|_2 \le \frac{ \sqrt{\rho_u^2+\rho_w^2} }{\sqrt{m}},
\qquad
\|\delta\theta_i\|_2^2 \le \frac{ \rho_u^2+\rho_w^2 }{m}.
\]
Using Lemma~\ref{lem:h-local-Lipschitz-smooth} and the fact that $|c_i^{(0)}|= 1$ at initialization, on the event $\mathcal{E}_{U,m}(\delta')$ we obtain
\begin{align*}
|e_{\mathrm{lin}}(X)|
&\le \frac{1}{\sqrt{m}}\sum_{i=1}^m
      \frac{\rho_c}{\sqrt{m}}\cdot L_{1,m}(\delta')\cdot \frac{\sqrt{\rho_u^2+\rho_w^2}}{\sqrt{m}}
   \;+\; \frac{1}{\sqrt{m}}\sum_{i=1}^m
      1\cdot \frac{L_{2,m}(\delta')}{2}\cdot \frac{\rho_u^2+\rho_w^2}{m}\\
&= \frac{1}{\sqrt{m}}\, L_{1,m}(\delta')\, \rho_c\, \sqrt{\rho_u^2+\rho_w^2}
   \;+\; \frac{1}{\sqrt{m}}\, \frac{L_{2,m}(\delta')}{2}\, (\rho_u^2+\rho_w^2)\\
&\le \frac{1}{\sqrt{m}}
\Big( L_{1,m}(\delta')\, \rho_c\, \sqrt{\rho_u^2+\rho_w^2}
     + L_{2,m}(\delta')\, (\rho_u^2+\rho_w^2) \Big).
\end{align*}
This bound holds for every $X=X^{(j)}$, hence after taking a supremum over $j\in[n]$ we get
\[
\varepsilon_{\mathrm{lin}} \;=\; \sup_{j\in[n]} |e_{\mathrm{lin}}(X^{(j)})|
\;\le\; \frac{1}{\sqrt{m}}
\Big( L_{1,m}(\delta')\, \rho_c\, \sqrt{\rho_u^2+\rho_w^2}
     + L_{2,m}(\delta')\, (\rho_u^2+\rho_w^2) \Big)
\;=: B_{\mathrm{lin},m}(\delta') .
\]
\end{proof}

\subsection{Proof of Lemma~\ref{lem:ntk}}

\begin{lemma}
Let $a=a(X;W)$, $a'=a(X';W)$, $M=X\,J_s(X^\top W q_X)\,X^\top$, and
$M'=X'\,J_s(X'^\top W q_{X'})\,X'^\top$. Then, with expectation taken over the random
initialization $(c,U,\operatorname{vec}(W)) \sim \varphi^{(0)}_1$,
\begin{align*}
K_c(X,X') &= \mathbb{E}\!\left[\,\sigma(U^\top a)\,\sigma(U^\top a')\,\right],\\
K_u(X,X') &= \mathbb{E}\!\left[\,\sigma'(U^\top a)\,\sigma'(U^\top a')\,\langle a,a'\rangle\,\right],\\
K_w(X,X') &= \langle q_X, q_{X'}\rangle
\mathbb{E}\!\left[\sigma'(U^\top a)\sigma'(U^\top a')U^\top M M' U\right]
\end{align*}
and the NTK decomposes as $K=K_c+K_u+K_w$.
\end{lemma}

\begin{proof}
Define the width-$m$ neural tangent kernel at initialization as
\[
K_m(X,X') \;\coloneqq\; \big\langle \nabla_\varphi f\!\big(X;\varphi^{(0)}\big),\, \nabla_\varphi f\!\big(X';\varphi^{(0)}\big)\big\rangle.
\]
By block structure, this decomposes as a sum over units and parameter blocks:
\begin{multline*}
K_m(X,X') \;=\; \sum_{i=1}^m\!\Biggl[
\underbrace{\frac{\partial f}{\partial c_i}(X; \varphi^{(0)})\,\frac{\partial f}{\partial c_i}(X'; \varphi^{(0)})}_{\text{$c$–block}}
\;+\; \underbrace{\big\langle \nabla_{U_i} f(X; \varphi^{(0)}),\, \nabla_{U_i} f(X'; \varphi^{(0)}) \big\rangle}_{\text{$U$–block}}
\\ \;+\; \underbrace{\big\langle \nabla_{W_i} f(X; \varphi^{(0)}),\, \nabla_{W_i} f(X'; \varphi^{(0)}) \big\rangle_F}_{\text{$W$–block}}
\Biggr].
\end{multline*}

We remove the superscript and write $(c_i^{(0)}, U_i^{(0)}, W_i^{(0)}) = (c_i, U_i, W_i)$
\paragraph{$c$–component.}
\begin{align*}
\sum_{i=1}^m \frac{\partial f}{\partial c_i}(X; \varphi^{(0)})\,\frac{\partial f}{\partial c_i}(X'; \varphi^{(0)})
&= \frac{1}{m}\sum_{i=1}^m \sigma\!\big({U_i}^\top a(X;W_i)\big)\, \sigma\!\big(U_i^\top a(X';W_i)\big) \\
&= \frac{2}{m}\sum_{i=1}^{m/2} \sigma\!\big({U_i}^\top a(X;W_i)\big)\, \sigma\!\big(U_i^\top a(X';W_i)\big).
\end{align*}
Since the pairs $(c_i,U_i,W_i)_{i=1}^{m/2}$ are i.i.d.\ at initialization, by the strong law of large numbers (SLLN),
\[
K_c(X,X') \;:=\; \lim_{m\to\infty} \frac{2}{m}\sum_{i=1}^{m/2} \sigma\!\big(U_i^\top a(X;W_i)\big)\, \sigma\!\big(U_i^\top a(X';W_i)\big)
= \mathbb{E}\!\left[\,\sigma(U^\top a)\,\sigma(U^\top a')\,\right].
\]

\paragraph{$U$–component.}
Writing $a_i=a(X;W_i)$ and $a_i'=a(X';W_i)$,
\begin{align*}
\sum_{i=1}^m \big\langle \nabla_{U_i} f(X;\varphi^{(0)}),\, \nabla_{U_i} f(X';\varphi^{(0)}) \big\rangle
&= \frac{1}{m}\sum_{i=1}^m c_i^2\, \sigma'\!\big(U_i^\top a_i\big)\,\sigma'\!\big(U_i^\top a_i'\big)\; \langle a_i,a_i'\rangle \\
&= \frac{2}{m}\sum_{i=1}^{m/2} \sigma'\!\big(U_i^\top a_i\big)\,\sigma'\!\big(U_i^\top a_i'\big)\; \langle a_i,a_i'\rangle \\
\end{align*}
where we used the fact that $c_i^2=1$ surely at initialization. Then SLLN yields
\[
K_u(X,X') \;:=\; \lim_{m\to\infty} \frac{2}{m}\sum_{i=1}^{m/2} \sigma'\!\big(U_i^\top a_i\big)\,\sigma'\!\big(U_i^\top a_i'\big)\; \langle a_i,a_i'\rangle
= \mathbb{E}\!\left[\,\sigma'(U^\top a)\,\sigma'(U^\top a')\,\langle a,a'\rangle\,\right].
\]

\paragraph{$W$–component.}
Writing $M_i=X\,J_s(X^\top W_i q_X)\,X^\top$ and
$M_i'=X'\,J_s(X'^\top W_i q_{X'})\,X'^\top$, using the Frobenius product identity $\langle xy^\top, x'y'^\top\rangle_F=\langle x,x'\rangle\,\langle y,y'\rangle$ and that $M_i$ and $M_i'$ are symmetric,
\begin{align*}
\sum_{i=1}^m \big\langle \nabla_{W_i} f(X),\, \nabla_{W_i} f(X') \big\rangle_F
&= \frac{1}{m}\sum_{i=1}^m \sigma'\!\big(U_i^\top a_i\big)\sigma'\!\big(U_i^\top a_i'\big)
   \big\langle (M_i U_i) q_X^\top,\; (M_i' U_i) q_{X'}^\top \big\rangle_F \\
&= \frac{2}{m}\sum_{i=1}^{m/2} \sigma'\!\big(U_i^\top a_i\big)\sigma'\!\big(U_i^\top a_i'\big)
   \;\langle q_X,q_{X'}\rangle\; U_i^\top M_i M_i' U_i.
\end{align*}
Again using SLLN,
\begin{align*}
K_w(X,X') &:= \lim_{m\to\infty} \frac{2}{m}\sum_{i=1}^{m/2} \sigma'\!\big(U_i^\top a_i\big)\sigma'\!\big(U_i^\top a_i'\big)\,
   \langle q_X,q_{X'}\rangle\, U_i^\top M_i M_i' U_i \\
&= \langle q_X,q_{X'}\rangle\,\mathbb{E}\!\left[ \sigma'(U^\top a)\sigma'(U^\top a')\, U^\top M M' U \right].
\end{align*}

Putting the three blocks together and taking $m\to\infty$,
\[
K(X,X') \;=\; \lim_{m\to\infty} K_m(X,X') \;=\; K_c(X,X') + K_u(X,X') + K_w(X,X').
\]
\end{proof}

\subsection{Proof of Lemma~\ref{lem:approximation}}

\begin{lemma}
Let $\tilde{f}(\cdot; v) \in \mathcal{F}_{\bar{\nu}}$ and $\tilde{\varphi}$ be as in \eqref{eq:good-params}. Define 
\[\varepsilon_{\mathrm{app}}:=\sup_{j\in[n]}\ \big|f_{\mathrm{lin}}(X^{(j)};\tilde\varphi) - \tilde f(X^{(j)};v)\big|.\] 
Then for any $\delta \in (0,1)$, we have
\[
\varepsilon_{\mathrm{app}}
\;\leq\; 4(\sigma_0\bar{\nu}_c + \sigma_1\bar{\nu}_u+\sigma_1\bar{\nu}_w)\sqrt{\tfrac{\log(2n/\delta)}{m}} =: B_{\mathrm{app}, m}(\delta),
\]
with probability at least \(1-\delta\). We denote this event with $\mathcal E_{\mathrm{app}}(\delta)$.
\end{lemma}

\begin{proof}
Fix \(X\in\mathcal X\), and write $a:=a(X;W)$, $M:=M(X;W)$. Define, for each \(i\in[m]\),
\[
g_i(X):=g_i^{(c)}(X)+g_i^{(u)}(X)+g_i^{(w)}(X),
\]
where
\begin{align*}
g_i^{(c)}(X)
&:= \phi_c\!\big(X;\varphi_i^{(0)}\big)\,v_c\!\big(\varphi_i^{(0)}\big),\\
g_i^{(u)}(X)
&:= \Big\langle \phi_u\!\big(X;\varphi_i^{(0)}\big),\,v_u\!\big(\varphi_i^{(0)}\big)\Big\rangle,\\
g_i^{(w)}(X)
&:= \Big\langle \phi_w\!\big(X;\varphi_i^{(0)}\big),\,v_w\!\big(\varphi_i^{(0)}\big)\Big\rangle_F,
\end{align*}
with
\[
\phi_c(X;\varphi_0)=\sigma(U^\top a(X;W)),\qquad
\phi_u(X;\varphi_0)=\sigma'(U^\top a(X;W))\,a(X;W),
\]
and
\[
\phi_w(X;\varphi_0)=\sigma'(U^\top a(X;W))\,(M(X;W)U)\,q_X^\top.
\]
By the definition of \(f_{\mathrm{lin}}\) and \(\tilde f(\cdot;v)\),
\[
f_{\mathrm{lin}}(X;\tilde\varphi)=\frac1m\sum_{i=1}^m g_i(X),
\qquad
\tilde f(X;v)=\mathbb E[g_1(X)].
\]
Hence
\[
f_{\mathrm{lin}}(X;\tilde\varphi)-\tilde f(X;v)
=
\frac1m\sum_{i=1}^m\big(g_i(X)-\mathbb E[g_i(X)]\big).
\]

Using the paired initialization, define
\[
\bar g_i^{(\ell)}(X):=\frac12\Big(g_i^{(\ell)}(X)+g_{i+m/2}^{(\ell)}(X)\Big),
\qquad \ell\in\{c,u,w\},\quad i\in[m/2].
\]
Since the pairs \(\big(\varphi_i^{(0)},\varphi_{i+m/2}^{(0)}\big)\) are i.i.d. across \(i\in[m/2]\), the random variables
\[
Z_i(X):=
\sum_{\ell\in\{c,u,w\}}
\Big(\bar g_i^{(\ell)}(X)-\mathbb E[\bar g_1^{(\ell)}(X)]\Big),
\qquad i\in[m/2],
\]
are i.i.d. and mean zero, and
\[
f_{\mathrm{lin}}(X;\tilde\varphi)-\tilde f(X;v)
=
\frac2m\sum_{i=1}^{m/2} Z_i(X).
\]

\paragraph{\(c\)-component.}
Since \(|\phi_c|\le \sigma_0\) and \(|v_c|\le \bar\nu_c\), we have
\[
|g_i^{(c)}(X)|\le \sigma_0\bar\nu_c,
\qquad
|\bar g_i^{(c)}(X)|\le \sigma_0\bar\nu_c.
\]
Therefore, by Hoeffding's lemma,
\[
\mathbb E\exp\!\Big(\lambda\big(\bar g_i^{(c)}(X)-\mathbb E\bar g_i^{(c)}(X)\big)\Big)
\le
\exp\!\Big(\frac{\lambda^2 (\sigma_0\bar\nu_c)^2}{2}\Big),
\qquad \forall \lambda\in\mathbb R.
\]

\paragraph{\(u\)-component.}
By Fact~\ref{fact:a-bound}, \(\|a(X;W)\|_2\le 1\). Since \(|\sigma'|\le \sigma_1\) and \(\|v_u\|_2\le \bar\nu_u\),
\[
|g_i^{(u)}(X)|
\le
\|\phi_u(X;\varphi_i^{(0)})\|_2\,\|v_u(\varphi_i^{(0)})\|_2
\le
\sigma_1\bar\nu_u,
\]
and hence \(|\bar g_i^{(u)}(X)|\le \sigma_1\bar\nu_u\). Another application of Hoeffding's lemma gives
\[
\mathbb E\exp\!\Big(\lambda\big(\bar g_i^{(u)}(X)-\mathbb E\bar g_i^{(u)}(X)\big)\Big)
\le
\exp\!\Big(\frac{\lambda^2 (\sigma_1\bar\nu_u)^2}{2}\Big),
\qquad \forall \lambda\in\mathbb R.
\]

\paragraph{\(w\)-component.}
Let $Y_i := \bar g_i^{(w)}(X)$. We have
$|Y_i| \le \sigma_1\bar\nu_w \|M(X;W_i^{(0)}) U_i^{(0)}\|_2$.
We claim that
\[
\mathbb{E}\exp\!\Big(\lambda(Y_i - \mathbb{E}Y_i)\Big) \le \exp(2\lambda^2 (\sigma_1\bar\nu_w)^2), \qquad \forall\,\lambda \in \mathbb{R}.
\]

By standard symmetrization trick, it suffices to show
\begin{equation}\label{eq:Yi-moment}
\mathbb{E}|Y_i|^{2k} \le (\sigma_1\bar\nu_w)^{2k} \cdot \frac{(2k)!}{2^k k!}, \qquad k \ge 1.
\end{equation}

Conditional on $W_i^{(0)}$, let $s_1,\dots,s_d$ be the singular values of $M(X;W_i^{(0)})$.
Since $U_i^{(0)} \sim \mathcal{N}(0,I_d)$, we have
$\|M U_i^{(0)}\|_2^2 \stackrel{d}{=} \sum_{r=1}^d s_r^2 G_r^2$
with $G_r \stackrel{\mathrm{i.i.d.}}{\sim} \mathcal{N}(0,1)$.
By Jensen's inequality (viewing $\{s_r^2/\sum s_r^2\}$ as weights and using convexity of $x \mapsto x^k$),
\[
\Big(\sum_{r} s_r^2 G_r^2\Big)^k
\le \Big(\sum_{r} s_r^2\Big)^{k-1} \sum_{r} s_r^2 G_r^{2k}.
\]
Taking conditional expectation and using $\sum_r s_r^2 = \mathrm{tr}(M^\top M) \le (\mathrm{tr}\,M)^2 \le 1$ (Fact~\ref{fact:cov-bound}), we get
\[
\mathbb{E}\big[\|M U_i^{(0)}\|_2^{2k} \mid W_i^{(0)}\big]
\le \mathbb{E}|G_1|^{2k} = \frac{(2k)!}{2^k k!},
\]
which gives \eqref{eq:Yi-moment} after multiplying by $(\sigma_1\bar\nu_w)^{2k}$ and taking full expectation.

Substituting \eqref{eq:Yi-moment} into the symmetrized MGF series yields
\[
\mathbb{E}\,e^{\lambda(Y_i - \mathbb{E}Y_i)}
\le \sum_{k=0}^\infty \frac{(2\lambda \sigma_1\bar\nu_w)^{2k}}{(2k)!} \cdot \frac{(2k)!}{2^k k!}
= \sum_{k=0}^\infty \frac{(2\lambda^2 (\sigma_1\bar\nu_w)^2)^k}{k!}
= e^{2\lambda^2 (\sigma_1\bar\nu_w)^2},
\]
proving the claim.

\paragraph{Combining the bounds.}
Write
\[
Z_i^{(c)}:=\bar g_i^{(c)}-\mathbb E\bar g_i^{(c)},
\qquad
Z_i^{(u)}:=\bar g_i^{(u)}-\mathbb E\bar g_i^{(u)},
\qquad
Z_i^{(w)}:=\bar g_i^{(w)}-\mathbb E\bar g_i^{(w)}.
\]
Then \(Z_i=Z_i^{(c)}+Z_i^{(u)}+Z_i^{(w)}\). For any \(\lambda\in\mathbb R\), applying Hölder's inequality with exponents \(1/p_c+1/p_u+1/p_w=1\) we get
\begin{align*}
\mathbb E e^{\lambda Z_i}
&\le
\Big(\mathbb E e^{p_c\lambda Z_i^{(c)}}\Big)^{1/p_c}
\Big(\mathbb E e^{p_u\lambda Z_i^{(u)}}\Big)^{1/p_u}
\Big(\mathbb E e^{p_w\lambda Z_i^{(w)}}\Big)^{1/p_w}\\
&\le
\exp\!\Big(\frac{p_c\lambda^2 (\sigma_0\bar\nu_c)^2}{2}\Big)
\exp\!\Big(\frac{p_u\lambda^2 (\sigma_1\bar\nu_u)^2}{2}\Big)
\exp\!\Big(2p_w\lambda^2 (\sigma_1\bar\nu_w)^2\Big)\\
&=
\exp\!\Big(\frac{\lambda^2}{2}(\sigma_0\bar\nu_c+\sigma_1\bar\nu_u+2\sigma_1\bar\nu_w)^2\Big),
\end{align*}
where in the last equality we set
\[
p_c=\frac{\sigma_0\bar\nu_c+\sigma_1\bar\nu_u+2\sigma_1\bar\nu_w}{\sigma_0\bar\nu_c},\qquad p_u=\frac{\sigma_0\bar\nu_c+\sigma_1\bar\nu_u+2\sigma_1\bar\nu_w}{\sigma_1\bar\nu_u},\qquad p_w=\frac{\sigma_0\bar\nu_c+\sigma_1\bar\nu_u+2\sigma_1\bar\nu_w}{2\sigma_1\bar\nu_w}.
\]

Thus each \(Z_i\) is centered sub-Gaussian with proxy \(\sigma_0\bar\nu_c+\sigma_1\bar\nu_u+2\sigma_1\bar\nu_w\). Since \(Z_1,\dots,Z_N\) are i.i.d., for every \(\lambda\in\mathbb R\),
\[
\mathbb E\exp\!\Big(\lambda\sum_{i=1}^{m/2} Z_i\Big)
\le
\exp\!\Big(\frac{m\lambda^2}{4}(\sigma_0\bar\nu_c+\sigma_1\bar\nu_u+2\sigma_1\bar\nu_w)^2\Big).
\]
By Chernoff's bound,
\[
\mathbb P\!\left(
\left|\frac2m\sum_{i=1}^{m/2} Z_i\right|\ge t
\right)
\le
2\exp\!\left(
-\frac{m t^2}{4(\sigma_0\bar\nu_c+\sigma_1\bar\nu_u+2\sigma_1\bar\nu_w)^2}
\right).
\]
Choosing
\[
t=
\sqrt2\,(\sigma_0\bar\nu_c+\sigma_1\bar\nu_u+2\sigma_1\bar\nu_w)\sqrt{\frac{2\log(2n/\delta)}{m}},
\]
and a union bound over \(X^{(1)},\dots,X^{(n)}\) yields, with probability at least \(1-\delta\),
\[
\varepsilon_{\mathrm{app}}
=
\sup_{j\in[n]}
\left|f_{\mathrm{lin}}(X^{(j)};\tilde\varphi)-\tilde f(X^{(j)};v)\right|
\le
\sqrt2\,
(\sigma_0\bar\nu_c+\sigma_1\bar\nu_u+2\sigma_1\bar\nu_w)
\sqrt{\frac{2\log(2n/\delta)}{m}}.
\]
This proves the lemma.
\end{proof}

\subsection{Proof of Theorem~\ref{thm:projgd}} \label{ap:projgd}
The following lemma will be useful.
\begin{lemma}[Uniform envelope bounds]\label{lem:envelope}
On the event $\mathcal{E}_{U,m}(\delta')$ and for any $\varphi\in\Omega_\rho$,
\[
\sup_{X\in\mathcal{X}}|f(X;\varphi)|
\;\le\; L_{1,m}(\delta')\,\sqrt{\rho_u^2+\rho_w^2}\;+\;\sigma_0\,\rho_c =: f^{\mathrm{max}}.
\]
Moreover, for any $\tilde{f}\in\mathcal{F}_{\bar\nu}$,
\[
\sup_{X\in\mathcal{X}}|\tilde f(X;v)|
\;\le\; \sigma_0\,\bar\nu_c \;+\; \sigma_1\,\bar\nu_u \;+\; \sigma_1\,\sqrt{d}\,\bar\nu_w =: y^{\mathrm{max}}.
\]
\end{lemma}

\begin{proof}
\textbf{Bounding $|f(X;\varphi)|$.}
Fix $X \in \mathcal{X}$ and write $\delta c_i=c_i-c_i^{(0)}$, $\delta\theta_i=(U_i-U_i^{(0)},\operatorname{vec}(W_i-W_i^{(0)}))$.
Thanks to symmetric initialization, $f(X;\varphi^{(0)})=0$, hence
\[
|f(X;\varphi)|
= \underbrace{\frac{1}{\sqrt{m}}\sum_{i=1}^m \Big| c_i^{(0)}\big(h(X;\theta_i)-h(X;\theta_i^{(0)})\big)\Big|}_{(i)}
  + \underbrace{\frac{1}{\sqrt{m}}\sum_{i=1}^m \Big| \delta c_i\, h(X;\theta_i)\Big|}_{(ii)}.
\]
On $\mathcal{E}_{U,m}(\delta')$, Lemma~\ref{lem:h-local-Lipschitz-smooth} gives
$\big|h(X;\theta_i)-h(X;\theta_i^{(0)})\big|\le L_{1,m}(\delta')\|\delta\theta_i\|_2$ and we also have $|c_i^{(0)}| = 1$ surely. Using Cauchy--Schwarz and $\varphi\in\Omega_\rho$, we get
\[
(i)\leq\frac{1}{\sqrt{m}}\sum_{i=1}^m L_{1,m}(\delta')\|\delta\theta_i\|_2
\;\le\; L_{1,m}(\delta')\sqrt{\rho_u^2+\rho_w^2}.
\]
Similarly, since we have $|h(X;\theta_i)|\le \sigma_0$, we get
\[
(ii)
\;\le\; \frac{1}{\sqrt{m}}\sum_{i=1}^m \frac{\rho_c}{\sqrt{m}}\sigma_0
\;=\; \sigma_0\,\rho_c.
\]
Therefore,
$|f(X;\varphi)| \le L_{1,m}(\delta')\sqrt{\rho_u^2+\rho_w^2}+\sigma_0\rho_c$,
and taking the supremum over $X \in \mathcal{X}$ gives the stated bound.

\textbf{Bounding $|\tilde{f}(X;v)|$.}
Recall
\[
\tilde f(X;v)
= \mathbb{E}\!\left[\phi_c(X)\,v_c(\varphi_0)\right]
+ \mathbb{E}\!\left[\langle \phi_u(X), v_u(\varphi_0)\rangle\right]
+ \mathbb{E}\!\left[\langle \phi_w(X), v_w(\varphi_0)\rangle_F\right],
\]
with
$\phi_c(X)=\sigma(U^\top a)$,
$\phi_u(X)=\sigma'(U^\top a)\,a$,
$\phi_w(X)=\sigma'(U^\top a)\,(M U)\,q_X^\top$.
Using $|\sigma|\le\sigma_0$, $|\sigma'|\le\sigma_1$, $\|a\|_2\le 1$, and
$\|M\|_F\le 1$, we obtain
\begin{align*}
\big|\mathbb{E}[\phi_c(X)\,v_c(\varphi_0)]\big|
&\le \sigma_0\,\bar\nu_c,\\
\big|\mathbb{E}[\langle \phi_u(X), v_u(\varphi_0)\rangle]\big|
&\le \sigma_1\,\bar\nu_u,\\
\big|\mathbb{E}[\langle \phi_w(X), v_w(\varphi_0)\rangle_F]\big|
&\le \sigma_1\,\mathbb{E}\!\left[\|(M U)q_X^\top\|_F\right]\bar\nu_w
 \;\le\; \sigma_1\,\mathbb{E}\!\left[\|U\|_2\right]\bar\nu_w
 \;\le\; \sigma_1\,\sqrt{d}\,\bar\nu_w,
\end{align*}
where the last inequality uses $\mathbb{E}\|U\|_2\le\sqrt{d}$ for $U\sim\mathcal{N}(0,I_d)$.
Summing the three bounds and taking the supremum over $X \in \mathcal{X}$ proves the claim.

\end{proof}

\begin{theorem}
Assume $f^\star \in \mathcal{F}_{\bar\nu}$ and the projection radius satisfies $\rho \ge \bar\nu$.
Run \textsc{ProjGD} onto $\Omega_\rho$ for $\tau$ steps with the step size $\eta=1/\sqrt{\tau}$.
Then, conditioned on the event $\mathcal{E}_{U,m}(\delta') \cap \mathcal E_{\mathrm{app}}(\delta)$, we have
\begin{equation*}
\min_{s<\tau}\ \widehat{\mathcal{L}}_n\!\big(\varphi^{(s)}\big)\; \lesssim_{\bar{\sigma}, \bar{\nu}, \rho}
\! \;
\frac{\Lone^{4}}{\sqrt{\tau}} \\
+\Lone(B_{\mathrm{app}, m}(\delta) + B_{\mathrm{lin}, m}(\delta') + \varepsilon_{\mathrm{CoF}}),
\end{equation*}
and for the average iterate $\bar\varphi^{(\tau)}:=\frac{1}{\tau}\sum_{s=0}^{\tau-1}\varphi^{(s)}$,
\begin{equation*}
\widehat{\mathcal L}_n\!\big(\bar\varphi^{(\tau)}\big) \; \lesssim_{\bar{\sigma}, \bar{\nu}, \rho}
\! \;
\frac{\Lone^{4}}{\sqrt{\tau}} \\
+\Lone(B_{\mathrm{app}, m}(\delta) + B_{\mathrm{lin}, m}(\delta') + \varepsilon_{\mathrm{CoF}})
+B_{\mathrm{lin}, m}(\delta')^2,
\end{equation*}
where $B_{\mathrm{lin}, m}(\delta')$ and $B_{\mathrm{app}, m}(\delta)$ are defined as in Lemmas~\ref{lem:linearization} and \ref{lem:approximation}, and 
\[
\varepsilon_{\mathrm{CoF}} \; \lesssim_{\bar{\sigma}, \bar{\nu}, \rho} \; \frac{\Lone + \Ltwo}{\sqrt{m}}.
\]
\end{theorem}

\begin{proof}
Choose $v$ with $f^\star = \tilde{f}(\cdot;v)$ and set the Lyapunov function $\mathcal V(\varphi)=\|\varphi-\tilde\varphi\|_2^2$, where $\tilde\varphi$ is from \eqref{eq:good-params}. Write the pre-projection point 
$\varphi'^{(s)}=\varphi^{(s)}-\eta\,\nabla \widehat{\mathcal L}_n(\varphi^{(s)})$.
By non-expansivity of Euclidean projection onto a convex set and the Pythagorean identity,
\begin{equation}\label{eq:drift-raw}
\mathcal V(\varphi^{(s+1)}) \;\le\; \mathcal V(\varphi'^{(s)}) \;=\;
\mathcal V(\varphi^{(s)}) 
+ 2\eta\,\underbrace{\big\langle \nabla \widehat{\mathcal L}_n(\varphi^{(s)}),\,\tilde\varphi-\varphi^{(s)}\big\rangle}_{\text{(i)}}
+ \eta^2\,\underbrace{\big\|\nabla \widehat{\mathcal L}_n(\varphi^{(s)})\big\|_2^2}_{\text{(ii)}} .
\end{equation}

\paragraph{Controlling (i).}
For shorthand write $f^{(s)}_j:=f(X^{(j)};\varphi^{(s)})$, $f^\star_j:=f^\star(X^{(j)})$ and
$g^{(s)}_j:=\nabla_\varphi f(X^{(j)};\varphi^{(s)})$, $g^{(0)}_j:=\nabla_\varphi f(X^{(j)};\varphi^{(0)})$.
Then
\begin{align}
\big\langle \nabla \widehat{\mathcal L}_n(\varphi^{(s)}),\,\tilde\varphi-\varphi^{(s)}\big\rangle
&= \frac{2}{n}\sum_{j=1}^n \bigl(f^{(s)}_j-f^\star_j\bigr)\;\big\langle g^{(s)}_j,\,\tilde\varphi-\varphi^{(s)}\big\rangle . \label{eq:center-0}
\end{align}
Introduce the change-of-feature error
\begin{equation} \label{eq:cof-error}
\varepsilon^{\mathrm{CoF}}_j(\varphi^{(s)}):=\big\langle g^{(s)}_j-g^{(0)}_j,\,\tilde\varphi-\varphi^{(s)}\big\rangle,
\end{equation}
and note that
\begin{align}
\big\langle g^{(s)}_j,\,\tilde\varphi-\varphi^{(s)}\big\rangle
&= \big\langle g^{(0)}_j,\,\tilde\varphi-\varphi^{(s)}\big\rangle + \varepsilon^{\mathrm{CoF}}_j(\varphi^{(s)}) \notag\\
&= f_{\mathrm{lin}}(X^{(j)};\tilde\varphi) - f_{\mathrm{lin}}(X^{(j)};\varphi^{(s)}) + \varepsilon^{\mathrm{CoF}}_j(\varphi^{(s)}).
\label{eq:center-1}
\end{align}
Now add and subtract $f^\star_j - f^{(s)}_j$ in \eqref{eq:center-1} to obtain
\begin{equation}
\big\langle g^{(s)}_j,\,\tilde\varphi-\varphi^{(s)}\big\rangle
= \underbrace{f_{\mathrm{lin}}(X^{(j)};\tilde\varphi) - f^\star_j}_{\varepsilon^{\mathrm{app}}_j}
+ \bigl(f^\star_j - f^{(s)}_j\bigr) \notag
+ \underbrace{f^{(s)}_j - f_{\mathrm{lin}}(X^{(j)};\varphi^{(s)})}_{\varepsilon^{\mathrm{lin}}_j(\varphi^{(s)})}
+ \varepsilon^{\mathrm{CoF}}_j(\varphi^{(s)}).
\end{equation}

From Lemmas~\ref{lem:linearization} and \ref{lem:approximation} we have,
\begin{align}
\big|\varepsilon^{\mathrm{lin}}_j(\varphi^{(s)})\big|
&\le B_{\mathrm{lin},m}(\delta'), \label{eq:lin-bound-clean}\\[0.25em]
\big|\varepsilon^{\mathrm{app}}_j\big|
&\le B_{\mathrm{app},m}(\delta), \label{eq:app-bound-clean}
\end{align}

To bound $\big|\varepsilon^{\mathrm{CoF}}_j(\varphi^{(s)})\big|$, we start from \eqref{eq:cof-error} and use the triangle and Cauchy--Schwarz inequalities,
\begin{align}\label{eq:cof-cs}
\big|\varepsilon^{\mathrm{CoF}}_j(\varphi^{(s)})\big|
\;&\le\; \sum_{i=1}^m \Big|\big\langle \nabla_{\varphi_i} f(X^{(j)};\varphi^{(s)}) - \nabla_{\varphi_i} f(X^{(j)};\varphi^{(0)}),\; \tilde{\varphi}_i - \varphi^{(s)}_i \big\rangle\Big| \\
\;&\le\; \sum_{i=1}^m \big\|\nabla_{\varphi_i} f(X^{(j)};\varphi^{(s)}) - \nabla_{\varphi_i} f(X^{(j)};\varphi^{(0)})\big\|_2\, \big\|\tilde{\varphi}_i - \varphi^{(s)}_i\big\|_2.
\end{align}
By construction of $\tilde{\varphi}$ and $\Omega_\rho$,
\begin{equation}\label{eq:param-radius}
\big\|\tilde{\varphi}_i - \varphi^{(s)}_i\big\|_2
\;\le\; \frac{\|\bar{\nu}\|_2 + \|\rho\|_2}{\sqrt{m}}.
\end{equation}
Set $\delta c_i:=c_i^{(s)}-c_i^{(0)}$, $\delta\theta_i:=\theta_i^{(s)}-\theta_i^{(0)}$.
From Lemma~\ref{lem:grads},
\[
\nabla_{\theta_i} f(X;\varphi) \;=\; m^{-1/2}\, c_i\, \nabla_{\theta_i} h(X;\theta_i),
\qquad
\nabla_{c_i} f(X;\varphi) \;=\; m^{-1/2}\, h(X;\theta_i).
\]
Hence
\begin{align}
\nabla_{\theta_i} f\big(X^{(j)};\varphi^{(s)}\big) - \nabla_{\theta_i} f\big(X^{(j)};\varphi^{(0)}\big)
&= \frac{c_i^{(0)}}{\sqrt{m}}\Big(\nabla_{\theta_i} h\big(X^{(j)};\theta_i^{(s)}\big) - \nabla_{\theta_i} h\big(X^{(j)};\theta_i^{(0)}\big)\Big)
   + \frac{\delta c_i}{\sqrt{m}}\, \nabla_{\theta_i} h\big(X^{(j)};\theta_i^{(s)}\big), \label{eq:theta-diff}\\
\nabla_{c_i} f\big(X^{(j)};\varphi^{(s)}\big) - \nabla_{c_i} f\big(X^{(j)};\varphi^{(0)}\big)
&= \frac{1}{\sqrt{m}}\Big(h\big(X^{(j)};\theta_i^{(s)}\big)-h\big(X^{(j)};\theta_i^{(0)}\big)\Big). \label{eq:c-diff}
\end{align}
Invoking Lemma~\ref{lem:h-local-Lipschitz-smooth}, we have for all $i$:
\[
\big\|\nabla_{\theta_i} h(X;\theta_i^{(s)}) - \nabla_{\theta_i} h(X;\theta_i^{(0)})\big\|_2 \le L_{2,m}(\delta')\, \|\delta\theta_i\|_2,
\qquad
\big|h(X;\theta_i^{(s)}) - h(X;\theta_i^{(0)})\big| \le L_{1,m}(\delta')\, \|\delta\theta_i\|_2,
\]
and $\|\nabla_{\theta_i} h(X;\theta_i^{(s)})\|_2 \le L_{1,m}(\delta')$. With $|c_i^{(0)}|=1$ at initialization,
$|\delta c_i|\le \rho_c/\sqrt{m}$, and $\|\delta\theta_i\|_2 \le \sqrt{\rho_u^2+\rho_w^2}/\sqrt{m}$, the bounds
\eqref{eq:theta-diff}--\eqref{eq:c-diff} yield
\begin{align}
\big\|\nabla_{\theta_i} f(X^{(j)};\varphi^{(s)}) - \nabla_{\theta_i} f(X^{(j)};\varphi^{(0)})\big\|_2
&\le \frac{L_{2,m}(\delta')}{\sqrt{m}}\,\frac{\sqrt{\rho_u^2+\rho_w^2}}{\sqrt{m}} \notag
     \;+\; \frac{\rho_c}{\sqrt{m}}\,\frac{L_{1,m}(\delta')}{\sqrt{m}} \\ 
 &= \frac{L_{2,m}(\delta')\sqrt{\rho_u^2+\rho_w^2} + \rho_c L_{1,m}(\delta')}{m}, \label{eq:theta-block}\\
\big\|\nabla_{c_i} f(X^{(j)};\varphi^{(s)}) - \nabla_{c_i} f(X^{(j)};\varphi^{(0)})\big\|_2
&\le \frac{L_{1,m}(\delta')}{\sqrt{m}}\,\frac{\sqrt{\rho_u^2+\rho_w^2}}{\sqrt{m}}
 = \frac{L_{1,m}(\delta')\sqrt{\rho_u^2+\rho_w^2}}{m}. \label{eq:c-block}
\end{align}
Combining the $c$- and $\theta$-blocks via $\|(x,y)\|_2\le \|x\|_2+\|y\|_2$,
\begin{equation}\label{eq:block-sum}
\big\|\nabla_{\varphi_i} f(X^{(j)};\varphi^{(s)}) - \nabla_{\varphi_i} f(X^{(j)};\varphi^{(0)})\big\|_2
\;\le\; \frac{(L_{1,m}(\delta')+L_{2,m}(\delta'))\sqrt{\rho_u^2+\rho_w^2} + \rho_c L_{1,m}(\delta')}{m}.
\end{equation}
Plugging \eqref{eq:param-radius} and \eqref{eq:block-sum} into \eqref{eq:cof-cs}, we obtain, for every $j$,
\begin{equation}\label{eq:cof-bound-clean}
\big|\varepsilon^{\mathrm{CoF}}_j(\varphi^{(s)})\big|
\;\le\; \Big(\,(L_{1,m}(\delta')+L_{2,m}(\delta'))\sqrt{\rho_u^2+\rho_w^2} + \rho_c L_{1,m}(\delta')\,\Big)\,
\frac{\|\bar{\nu}\|_2 + \|\rho\|_2}{\sqrt{m}}
\;=:\; B_{\mathrm{CoF},m}(\delta').
\end{equation}

Putting \eqref{eq:center-1} in \eqref{eq:center-0} and using the bounds
\eqref{eq:lin-bound-clean}, \eqref{eq:app-bound-clean} and \eqref{eq:cof-bound-clean}, we have, for all $s\ge0$,
\begin{align}
\big\langle \nabla \widehat{\mathcal L}_n(\varphi^{(s)}),\,\tilde\varphi-\varphi^{(s)}\big\rangle
&= -\,2\,\widehat{\mathcal L}_n(\varphi^{(s)})
+ \frac{2}{n}\sum_{j=1}^n \bigl(f^{(s)}_j-f^\star_j\bigr)\Big(\varepsilon^{\mathrm{app}}_j + \varepsilon^{\mathrm{lin}}_j(\varphi^{(s)}) + \varepsilon^{\mathrm{CoF}}_j(\varphi^{(s)})\Big) \notag\\
&\le -\,2\,\widehat{\mathcal L}_n(\varphi^{(s)})
+ 2\,(f^{\max}+y^{\max})\Big(B_{\mathrm{app},m}(\delta)+B_{\mathrm{lin},m}(\delta')+B_{\mathrm{CoF},m}(\delta')\Big),
\label{eq:control-i-final}
\end{align}
where $f^{\max}$ and $y^{\max}$ are defined as in Lemma~\ref{lem:envelope}.

\paragraph{Controlling (ii).}
For any $j\in[n]$ and head $i\in[m]$,
\[
\Big\|\nabla_{\varphi_i} f\big(X^{(j)};\varphi^{(s)}\big)\Big\|_2
\;\le\; \frac{|h(X^{(j)};\theta_i^{(s)})|}{\sqrt{m}}
+\frac{|c_i^{(s)}|}{\sqrt{m}}\,
\Big\|\nabla_{\theta_i} h\big(X^{(j)};\theta_i^{(s)}\big)\Big\|_2
\;\le\; \frac{1}{\sqrt{m}}\Big(\sigma_0 + |c_i^{(s)}|\,L_{1,m}(\delta')\Big),
\]
and since $|c_i^{(s)}|\le 1+\rho_c/\sqrt{m}$ on $\Omega_\rho$,
\begin{equation}\label{eq:per-head-grad}
\Big\|\nabla_{\varphi_i} f\big(X^{(j)};\varphi^{(s)}\big)\Big\|_2
\;\le\; \frac{1}{\sqrt{m}}\Big(\sigma_0 + (1+\tfrac{\rho_c}{\sqrt{m}})\,L_{1,m}(\delta')\Big).
\end{equation}
Summing the squares over $i=1,\dots,m$ gives
\begin{equation}\label{eq:whole-grad-f}
\Big\|\nabla_{\varphi} f\big(X^{(j)};\varphi^{(s)}\big)\Big\|_2
\;\le\; A_{m}(\delta')
\qquad\text{with}\qquad
A_{m}(\delta'):=\sigma_0 + \Big(1+\tfrac{\rho_c}{\sqrt{m}}\Big)L_{1,m}(\delta').
\end{equation}
Therefore,
\begin{align}
\Big\|\nabla \widehat{\mathcal L}_n(\varphi^{(s)})\Big\|_2
&= \left\| \frac{2}{n}\sum_{j=1}^n \bigl(f^{(s)}_j-f^\star_j\bigr)\,
\nabla_{\varphi} f\big(X^{(j)};\varphi^{(s)}\big) \right\|_2 \notag\\
&\le \frac{2}{n}\sum_{j=1}^n |f^{(s)}_j-f^\star_j|\,\Big\|\nabla_{\varphi} f\big(X^{(j)};\varphi^{(s)}\big)\Big\|_2
\;\le\; 2\,(f^{\max}+y^{\max})\,A_{m}(\delta'), \notag
\end{align}
and hence
\begin{equation}\label{eq:Bgrad}
\Big\|\nabla \widehat{\mathcal L}_n(\varphi^{(s)})\Big\|_2^2
\;\le\; 4\,(f^{\max}+y^{\max})^2\,A_{m}(\delta')^2
\;=:\; B_{\mathrm{grad},m}(\delta').
\end{equation}

\paragraph{Lyapunov drift.}
Plugging \eqref{eq:control-i-final} and \eqref{eq:Bgrad} into \eqref{eq:drift-raw} yields
\begin{equation}\label{eq:drift-final}
\mathcal V(\varphi^{(s+1)})-\mathcal V(\varphi^{(s)})
\;\le\;
-4\eta\,\widehat{\mathcal L}_n(\varphi^{(s)})
+ 4\eta\,(f^{\max}+y^{\max})\big(B_{\mathrm{app},m}(\delta)+B_{\mathrm{lin},m}(\delta')+B_{\mathrm{CoF},m}(\delta')\big)
+ \eta^2 B_{\mathrm{grad},m}(\delta') .
\end{equation}
Summing \eqref{eq:drift-final} over $s=0,\dots,\tau-1$ and using
$\sum_{s<\tau}(\mathcal V(\varphi^{(s+1)})-\mathcal V(\varphi^{(s)}))
=\mathcal V(\varphi^{(\tau)})-\mathcal V(\varphi^{(0)})$,
we obtain after rearranging the terms
\begin{equation}\label{eq:avg-risk}
\frac{1}{\tau}\sum_{s=0}^{\tau-1}\widehat{\mathcal L}_n(\varphi^{(s)})
\;\le\;
\frac{\mathcal V(\varphi^{(0)})}{4\eta\tau}
+ (f^{\max}+y^{\max})\big(B_{\mathrm{app},m}(\delta)+B_{\mathrm{lin},m}(\delta')+B_{\mathrm{CoF},m}(\delta')\big)
+ \frac{\eta}{4}\,B_{\mathrm{grad},m}(\delta') .
\end{equation}
Choosing $\eta=\tau^{-1/2}$ and using $\mathcal{V}(\varphi^{(0)}) \leq \|\bar{\nu}\|_2^2$ gives
\begin{equation}\label{eq:min-risk-final}
\min_{0\le s<\tau}\widehat{\mathcal L}_n\!\big(\varphi^{(s)}\big)
\;\le\;
\frac{\| \bar{\nu} \|_2^2 + B_{\mathrm{grad},m}(\delta')}{4\sqrt{\tau}}
+ (f^{\max}+y^{\max})\big(B_{\mathrm{app},m}(\delta)+B_{\mathrm{lin},m}(\delta')+B_{\mathrm{CoF},m}(\delta')\big).
\end{equation}

\paragraph{Average iterate.} Define $\widehat{\mathcal L}_n^{\mathrm{lin}}(\varphi):=\frac{1}{n}\sum_{j=1}^n(f_{\mathrm{lin}}(X^{(j)};\varphi)-y^{(j)})^2$. From the uniform linearization bound in Lemma~\ref{lem:linearization}, for all $\varphi\in\Omega_\rho$,
\[
\widehat{\mathcal L}_n(\varphi)\le 2\,\widehat{\mathcal L}_n^{\mathrm{lin}}(\varphi)+2\,B_{\mathrm{lin},m}(\delta')^2,
\qquad
\widehat{\mathcal L}_n^{\mathrm{lin}}(\varphi)\le 2\,\widehat{\mathcal L}_n(\varphi)+2\,B_{\mathrm{lin},m}(\delta')^2.
\]
By convexity of $\widehat{\mathcal L}_n^{\mathrm{lin}}$ and Jensen,
\[
\widehat{\mathcal L}_n^{\mathrm{lin}}\!\big(\bar\varphi^{(\tau)}\big)
\;\le\;\frac1\tau\sum_{s=0}^{\tau-1}\widehat{\mathcal L}_n^{\mathrm{lin}}\!\big(\varphi^{(s)}\big)
\;\le\;\frac{2}{\tau}\sum_{s=0}^{\tau-1}\widehat{\mathcal L}_n\!\big(\varphi^{(s)}\big)+2\,B_{\mathrm{lin},m}(\delta')^2.
\]
Combining the two inequalities yields the “sandwich” inequality
\begin{equation} \label{eq:sandwich}
\widehat{\mathcal L}_n\!\big(\bar\varphi^{(\tau)}\big)
\;\le\;\frac{4}{\tau}\sum_{s=0}^{\tau-1}\widehat{\mathcal L}_n\!\big(\varphi^{(s)}\big)+6\,B_{\mathrm{lin},m}(\delta')^2.
\end{equation}
Substituting \eqref{eq:avg-risk} into \eqref{eq:sandwich} gives the desired inequality.
\end{proof}

\subsection{Proof of Proposition~\ref{thm:projsgd}}

\begin{proposition}
Under the same conditions as Theorem~\ref{thm:projgd}, we have, on the event $\mathcal{E}_{U,m}(\delta') \cap \mathcal E_{\mathrm{app}}(\delta)$,
\begin{equation*}
\mathbb{E}\left[\min_{s<\tau}\widehat{\mathcal L}_n(\varphi^{(s)})\,|\,\varphi^{(0)}\right]\; \lesssim_{\bar{\sigma}, \bar{\nu}, \rho}
\! \;
\frac{\Lone^{4}}{\sqrt{\tau}}
+\Lone(B_{\mathrm{app}, m}(\delta) + B_{\mathrm{lin}, m}(\delta') + \varepsilon_{\mathrm{CoF}}),
\end{equation*}
and for the average iterate,
\begin{equation*}
\mathbb{E}\left[\widehat{\mathcal L}_n\!\big(\bar\varphi^{(\tau)}\big)\,|\,\varphi^{(0)}\right]
\! \; \lesssim_{\bar{\sigma}, \bar{\nu}, \rho}
\! \;
\frac{\Lone^{4}}{\sqrt{\tau}} \\
+\Lone(B_{\mathrm{app}, m}(\delta) + B_{\mathrm{lin}, m}(\delta') + \varepsilon_{\mathrm{CoF}})
+B_{\mathrm{lin}, m}(\delta')^2.
\end{equation*}
\end{proposition}

\begin{proof}
Choose $v$ with $f^\star = \tilde{f}(\cdot;v)$ and define the Lyapunov function $\mathcal V(\varphi)=\|\varphi-\tilde\varphi\|_2^2$, where $\tilde\varphi$ is from \eqref{eq:good-params}. Write the pre‑projection iterate $\varphi'^{(s)}=\varphi^{(s)}-\eta\,\nabla_\varphi \ell_{J_s}(\varphi^{(s)})$.
By non‑expansivity of projection and the Pythagorean identity,
\begin{equation}\label{eq:psgd-drift}
\mathcal V(\varphi^{(s+1)}) \le \mathcal V(\varphi^{(s)}) 
+ 2\eta\underbrace{\langle \nabla_\varphi \ell_{J_s}(\varphi^{(s)}),\,\tilde\varphi-\varphi^{(s)}\rangle}_{\text{(i)}}
+ \eta^2\underbrace{\|\nabla_\varphi \ell_{J_s}(\varphi^{(s)})\|_2^2}_{\text{(ii)}}.
\end{equation}

\paragraph{Controlling (i).}
Let $E_s[\cdot]:=\mathbb{E}_{J_s}[\cdot\,|\,J_{<s}, \varphi^{(0)}]$ denote expectation w.r.t.\ the fresh index $J_s$, keeping the past $J_{<s}$ and $\varphi^{(0)}$ fixed (so $\varphi^{(s)}$ is fixed).
Since $E_s\big[\nabla_\varphi \ell_{J_s}(\varphi^{(s)})\big]=\nabla \widehat{\mathcal L}_n(\varphi^{(s)})$, the same “Controlling (i)” computation as in the proof of Theorem~\ref{thm:projgd} gives
\[
E_s\big[\langle \nabla_\varphi \ell_{J_s}(\varphi^{(s)}),\,\tilde\varphi-\varphi^{(s)}\rangle\big]
\le -\,2\,\widehat{\mathcal L}_n(\varphi^{(s)}) + 2\,(f^{\max}+y^{\max})\Big(B_{\mathrm{app},m}(\delta)+B_{\mathrm{lin},m}(\delta')+B_{\mathrm{CoF},m}(\delta')\Big).
\]

\paragraph{Controlling (ii).}
Similarly, using the same arguments leading to \eqref{eq:Bgrad}, we have
\(
\|\nabla_\varphi \ell_{J_s}(\varphi^{(s)})\|_2^2\le B_{\mathrm{grad},m}(\delta')
\)
almost surely and therefore
\[
E_s\big[\|\nabla_\varphi \ell_{J_s}(\varphi^{(s)})\|_2^2\big]\le B_{\mathrm{grad},m}(\delta').
\]

\paragraph{Lyapunov drift.}
Taking $E_s[\cdot]$ in \eqref{eq:psgd-drift} and combining the previous two steps yields
\[
E_s\big[\mathcal V(\varphi^{(s+1)})-\mathcal V(\varphi^{(s)})\big]
\le -4\eta\,\widehat{\mathcal L}_n(\varphi^{(s)})
+ 4\eta\,(f^{\max}+y^{\max})\Big(B_{\mathrm{app},m}(\delta)+B_{\mathrm{lin},m}(\delta')+B_{\mathrm{CoF},m}(\delta')\Big)
+ \eta^2 B_{\mathrm{grad},m}(\delta').
\]
Taking expectation over $J_0,\ldots,J_{\tau-1}$, summing $s=0,\dots,\tau-1$, using 
$\sum_{s<\tau}\big(\mathcal V(\varphi^{(s+1)})-\mathcal V(\varphi^{(s)})\big)
=\mathcal V(\varphi^{(\tau)})-\mathcal V(\varphi^{(0)})$, and $\mathcal{V}(\varphi^{(0)}) \leq \|\bar{\nu}\|_2^2$, we obtain after rearranging the terms
\begin{equation} \label{eq:average-loss-sgd}
\frac{1}{\tau}\sum_{s=0}^{\tau-1}\mathbb{E}\big[\widehat{\mathcal L}_n(\varphi^{(s)}) | \varphi^{(0)} \big]
\le \frac{\|\bar\nu\|_2^2}{4\eta\tau}
+ (f^{\max}+y^{\max})\Big(B_{\mathrm{app},m}(\delta)+B_{\mathrm{lin},m}(\delta')+B_{\mathrm{CoF},m}(\delta')\Big)
+ \frac{\eta}{4}B_{\mathrm{grad},m}(\delta').
\end{equation}
With $\eta=\tau^{-1/2}$ and $\min_{0\le s<\tau}\widehat{\mathcal L}_n(\varphi^{(s)})
\le \tfrac1\tau\sum_{s<\tau}\widehat{\mathcal L}_n(\varphi^{(s)})$, we get
\begin{equation}
\mathbb{E}\left[\min_{0\le s<\tau}\widehat{\mathcal L}_n\!\big(\varphi^{(s)}\big) | \varphi^{(0)}\right]
\;\le\;
\frac{\|\bar\nu\|_2^2 + B_{\mathrm{grad},m}(\delta')}{4\sqrt{\tau}}
+ (f^{\max}+y^{\max})\Big(B_{\mathrm{app},m}(\delta)+B_{\mathrm{lin},m}(\delta')+B_{\mathrm{CoF},m}(\delta')\Big).
\end{equation}
\paragraph{Average iterate.} Taking conditional expectation given $\varphi^{(0)}$ on both sides of \eqref{eq:sandwich}, we get
\begin{equation} \label{eq:sandwich-sgd}
\mathbb{E}\!\left[\widehat{\mathcal L}_n\!\big(\bar\varphi^{(\tau)}\big)\,\middle|\,\varphi^{(0)}\right]
\;\le\;
\frac{4}{\tau}\sum_{s=0}^{\tau-1}\mathbb{E}\!\left[\widehat{\mathcal L}_n\!\big(\varphi^{(s)}\big)\,\middle|\,\varphi^{(0)}\right]
+6\,B_{\mathrm{lin},m}(\delta')^2.
\end{equation}
Substituting \eqref{eq:average-loss-sgd} into \eqref{eq:sandwich-sgd} gives the desired inequality.
\end{proof}

\section{EXPERIMENTAL DETAILS} \label{sec:app-exp-details}
\paragraph{Code and Reproducibility.}
All code to reproduce our experiments is available at the following repository:
\[
\texttt{\url{https://github.com/enesarda22/nonasymptotic-transformer}}
\]

\subsection{IndRNN Architecture}
Our recurrent model is an Independent RNN (IndRNN) studied by \citet{cayci2024rnngd} with one key difference: because the prediction target in our case is a scalar, we use only the final hidden state and apply a linear layer, rather than producing a sequence of outputs. Concretely, for an input sequence
\(
X=[X_1,\ldots,X_T]\in\mathbb{R}^{d\times T},
\)
width $m$, activation $\sigma$, hidden matrix $U\in\mathbb{R}^{m\times d}$, diagonal recurrent weights $w\in\mathbb{R}^m$, and linear layer weights $c\in\mathbb{R}^m$, we define
\begin{align}
h_0 &= 0\in\mathbb{R}^m,\\
h_t &= \sigma\!\big(U\,X_t + w\odot h_{t-1}\big)\in\mathbb{R}^m,\qquad t=1,\ldots,T,\\
f_{\mathrm{rnn}}(X;\varphi_{\mathrm{rnn}}) &:= \frac{1}{\sqrt{m}}\,c^\top h_T\in\mathbb{R},
\end{align}
where $\odot$ denotes elementwise multiplication and $\varphi_{\mathrm{rnn}}=(\varphi_{\mathrm{rnn}, i})_{i=1}^m$ with $\varphi_{\mathrm{rnn}, i}:=(U_i,w_i,c_i)$.

\paragraph{Symmetric initialization.}
Similar to Section~\ref{sec:init}, for IndRNN, we use a symmetric initialization: assume $m$ is even and draw the first half independently as
\begin{equation} \label{eq:rnn-init}
U_i^{(0)}\sim\mathcal{N}(0,I_d),\qquad
w_i^{(0)}\sim\mathrm{Rad}(\gamma),\quad
c_i^{(0)}\sim\mathrm{Rad}(1),\qquad i=1,\ldots,\tfrac{m}{2}.
\end{equation}
The second half is set as
\[
U_{i+m/2}^{(0)}=U_i^{(0)},\quad w_{i+m/2}^{(0)}=w_i^{(0)},\quad c_{i+m/2}^{(0)}=-\,c_i^{(0)}.
\]
This pairing ensures that the initial predictor is identically zero for any input $X \in \mathcal{X}$.

\paragraph{Parameter set.}
Similar to Definition~\ref{def:parameter-set}, for IndRNN, we define the neighborhood for projection to be
\[
\Omega_{\mathrm{rnn},\rho}
:=\Big\{(U,w,c):\ \|U-U^{(0)}\|_{2,\text{row}}\le \tfrac{\rho_u}{\sqrt m},\ 
\|w-w^{(0)}\|_\infty\le \tfrac{\rho_w}{\sqrt m},\ 
\|c-c^{(0)}\|_\infty\le \tfrac{\rho_c}{\sqrt m}\Big\}.
\]

\paragraph{Training error for IndRNN.}
For completeness, we state the training loss result for the IndRNN architecture with scalar outputs, which follows as a corollary of \citep[Thm.~4.4]{cayci2024rnngd}.

\begin{proposition}
Assume $\tilde{f} \in \mathcal{F}_{\mathrm{rnn},\bar\nu}$ (see \citep[Eq.~(3.5)]{cayci2024rnngd}), and that the targets are generated as $y^{(j)}=\tilde{f}(X^{(j)})_T$ for $j=1,\dots,n$, with projection radius $\rho \ge \bar\nu$.
Run \textsc{ProjGD} on $f_{\mathrm{rnn}}$ with projection onto $\Omega_{\mathrm{rnn},\rho}$ for $\tau$ steps with step size $\eta=1/\sqrt{\tau}$.
Then, for any $\delta \in (0,1)$, with probability at least $1-\delta$,
\[
\min_{s<\tau}\ \widehat{\mathcal{L}}_n\!\big(\varphi_\mathrm{rnn}^{(s)}\big)\;\lesssim_{\bar{\sigma},\bar{\nu},\rho,\mu_T}\; \frac{1}{\sqrt{\tau}} \;+\; \sqrt{\frac{\log(2n/\delta)}{m}},
\]
where $\mu_T=\mathcal{O}(1)$ if $\gamma+\tfrac{\rho_w}{\sqrt{m}} \le \tfrac{1}{\sigma_1}$, and $\mu_T=\exp(\Omega(T))$ otherwise. Here $\gamma>0$ is the Rademacher scale in the symmetric initialization of the diagonal recurrent weights (see \eqref{eq:rnn-init}).
\end{proposition}

\paragraph{Max Jacobian norms in training.}
For Transformer and IndRNN we calculate
\[
\mathcal{G}_{\mathrm{rnn}}(\varphi_\mathrm{rnn})
:= \frac{1}{n}\sum_{j=1}^n \nabla_w f_{\mathrm{rnn}}(X^{(j)};\varphi_{\mathrm{rnn}}),
\qquad
\mathcal{G}_{\mathrm{tf}}(\varphi)
:= \frac{1}{n}\sum_{j=1}^n \nabla_W f(X^{(j)};\varphi).
\]
In Figure~\ref{fig:rnn-vs-trf}, we plot $\max_{s<\tau}\|\mathcal{G}_{\mathrm{rnn}}(\varphi_{\mathrm{rnn}}^{(s)})
\|_2$ and $\max_{s<\tau}\|\mathcal{G}_{\mathrm{tf}}(\varphi^{(s)})\|_F$.

\end{document}